%% file: main.tex
\documentclass[10pt,twocolumn,letterpaper]{article}
\usepackage{templates/cvpr/cvpr} %

\newcommand{\ARXIV}[1]{#1} %

\usepackage{graphicx}
\usepackage{amsmath}
\usepackage{amssymb}
\usepackage{times}
\usepackage{mathtools}
\usepackage{mathrsfs}  
\usepackage{amssymb}
\usepackage{booktabs}
\usepackage{multirow,multicol}
\usepackage{array}
\usepackage{subcaption}
\usepackage{colortbl}
\usepackage{lipsum} 

\usepackage{fontawesome5}
\usepackage{enumitem}
\usepackage{tikz}
\usepackage{pgfplots}
\usepackage{microtype}
\usepackage{paralist}
\setdefaultleftmargin{1.5em}{}{}{}{}{}
\usepackage{enumitem}
\usepackage{adjustbox}
\usepackage{subcaption}
\usepackage{epigraph}
\usepackage{placeins}
\usepackage{amsthm}
\usepackage{balance}

\usepackage{custom_style}

\newtheorem{prop}{Proposition}

\usepackage[hang,flushmargin]{footmisc}
\usepackage[accsupp]{axessibility}  %
\usepackage[pagebackref,breaklinks,colorlinks]{hyperref}

\usetikzlibrary{tikzmark,patterns,shapes.misc,shapes.geometric,positioning,calc}

\usepackage[capitalize]{cleveref}
\crefname{section}{Sec.}{Secs.}
\Crefname{section}{Section}{Sections}
\Crefname{table}{Table}{Tables}
\crefname{table}{Tab.}{Tabs.}

\def\paperID{9} %
\def\confName{CVPR}
\def\confYear{2025}

\begin{document}

\title{
\hotairballoon
Around the World in 80 Timesteps:\\
A Generative Approach to Global Visual Geolocation}

\author{
    Nicolas Dufour \textsuperscript{1,2}
    \and
    David Picard \textsuperscript{1}
    \and
    Vicky Kalogeiton \textsuperscript{2}
    \and
    Loic Landrieu\textsuperscript{1}
    \and 
    {\textsuperscript{1} LIGM, Ecole des Ponts, IP Paris, CNRS, UGE}
    \and 
     {\textsuperscript{2} LIX, Ecole Polytechnique, IP Paris}
}

 \twocolumn[{%
 \renewcommand\twocolumn[1][]{#1}
 \maketitle
 \vspace{-4mm}
 \input{figures/teaser_large}
 \vspace{-2mm}
 \captionof{figure}{{\bf Geolocation as a Generative Process.} 
We explore diffusion and flow matching for visual geolocation by sampling and denoising random locations. This process generates trajectories onto the Earth's surface, whose endpoints provide location estimates. Our models also provide probability densities for every possible image locations. We illustrate these trajectories and the log-densities for three images from different datasets: an Andean condor from iNat21~\cite{van2021benchmarking}, an African  open-air market from YFCC-100M~\cite{YFCC}, and a dashcam snapshot from OSV-5M~\cite{astruc2024openstreetview}. The predicted image locations are indicated by \redcross~and the true ones by \greencross.\vspace{3mm}}
 \label{fig:teaser}
 }]

\begin{abstract}
Global visual geolocation consists in predicting where an image was captured anywhere on Earth. Since not all images can be localized with the same precision, this task inherently involves a degree of ambiguity. However, existing approaches are deterministic and overlook this aspect. In this paper, we propose the first generative approach for visual geolocation based on diffusion and flow matching, and an extension to Riemannian flow matching, where the denoising process operates directly on the Earth's surface. Our model achieves state-of-the-art performance on three visual geolocation benchmarks: OpenStreetView-5M, YFCC-100M, and iNat21. In addition, we introduce the task of probabilistic visual geolocation, where the model predicts a probability distribution over all possible locations instead of a single point. We implement new metrics and baselines for this task, demonstrating the advantages of our generative approach. Codes and models are available \href
{https://nicolas-dufour.github.io/plonk}{here}.

\end{abstract}

\section{Introduction}
\input{text/1_introduction}

\section{Related Work}

\input{text/2_related}

\section{Method}
\input{text/3_method}

\section{Experiments}

\input{text/4_experiments}

\section{Conclusion}
We introduced a novel generative approach to global visual geolocation based on diffusion models and Riemannian flow matching  on the Earth's surface. Our method effectively captures the inherent ambiguity in geolocating images---an aspect often overlooked by deterministic models. Experiments on three standard benchmarks demonstrated state-of-the-art geolocation performance. Additionally, we introduced the task of probabilistic visual geolocation, along with its metrics and baselines. Our generative approach predicts probability distributions that fits more closely to the data despite its high ambiguity. Our approach is especially valuable for applications involving images with vague or ambiguous location cues, where traditional methods struggle to provide meaningful predictions.
\section{Acknowledgements}
This work was supported by ANR project TOSAI ANR-20-IADJ-0009, and was granted access to the HPC resources of IDRIS under the allocation 2024-AD011015664 made by GENCI. We would like to thank Julie Mordacq, Elliot Vincent, and Yohann Perron for their helpful feedback.
\FloatBarrier
\balance{
\bibliographystyle{splncs04}
\bibliography{biblio}
}

\ARXIV{\input{text/A_suppmat}}{}

\end{document}

%% file: figures/teaser_large.tex
\begin{tabular}{c@{\hspace{1em}}c}
    \adjustbox{valign=m}{\includegraphics[width=.75\linewidth, height=.33\linewidth]{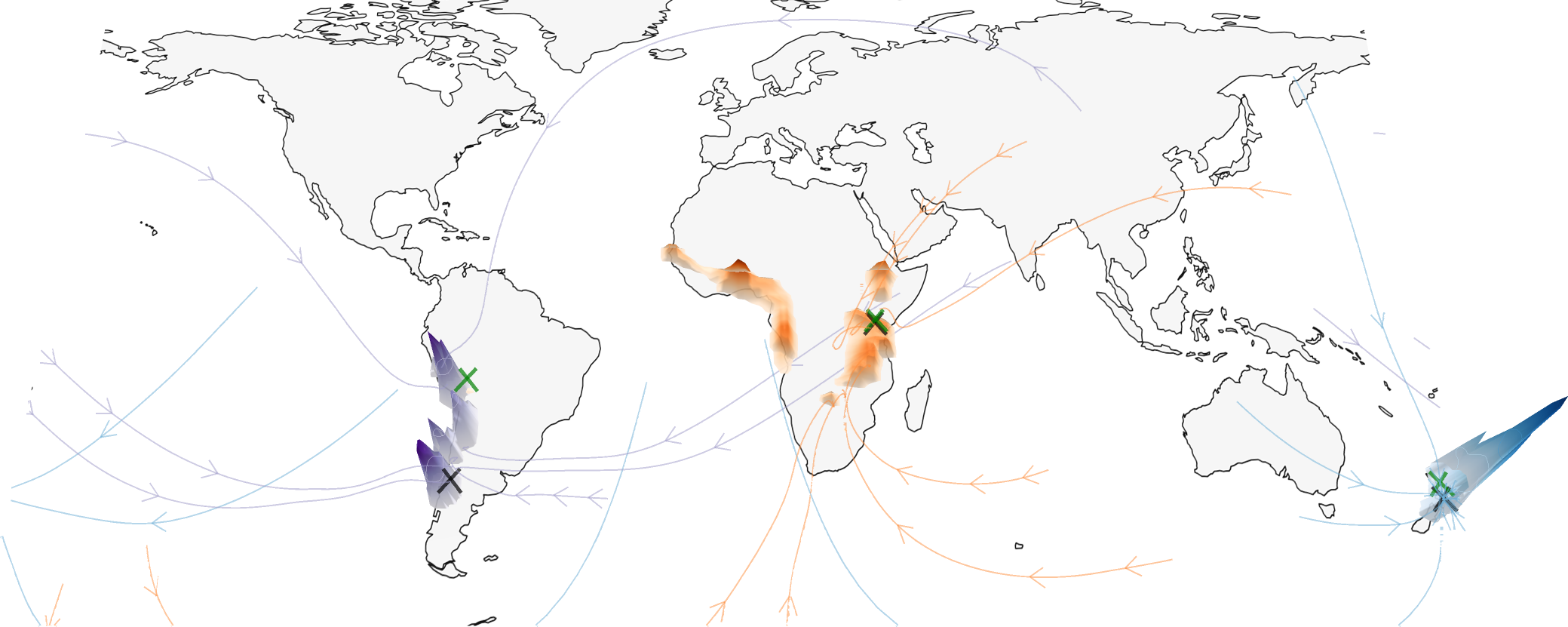}
    }
    & 
    \adjustbox{valign=m}{
        \begin{tabular}{@{}c@{}}
            \tikz \node[line width=2mm, draw=INATCOLOR, inner sep=0pt]{\includegraphics[width=.18\linewidth, height=.08\linewidth]{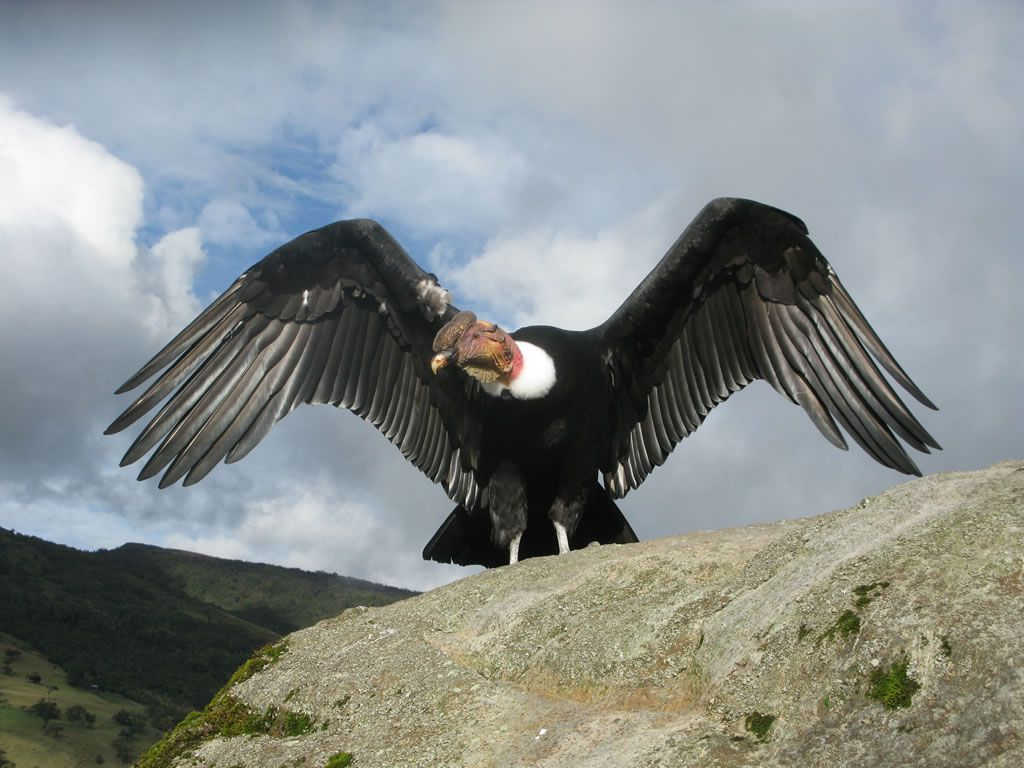}}; \\
            \tikz \node[inner sep=1pt, fill=none, text=black] {\small iNat-21 \cite{van2021benchmarking}};
             \\[0.0em]
            \tikz \node[line width=2mm, draw=YFCCCOLOR, inner sep=0pt]{\includegraphics[width=.18\linewidth, height=.08\linewidth]{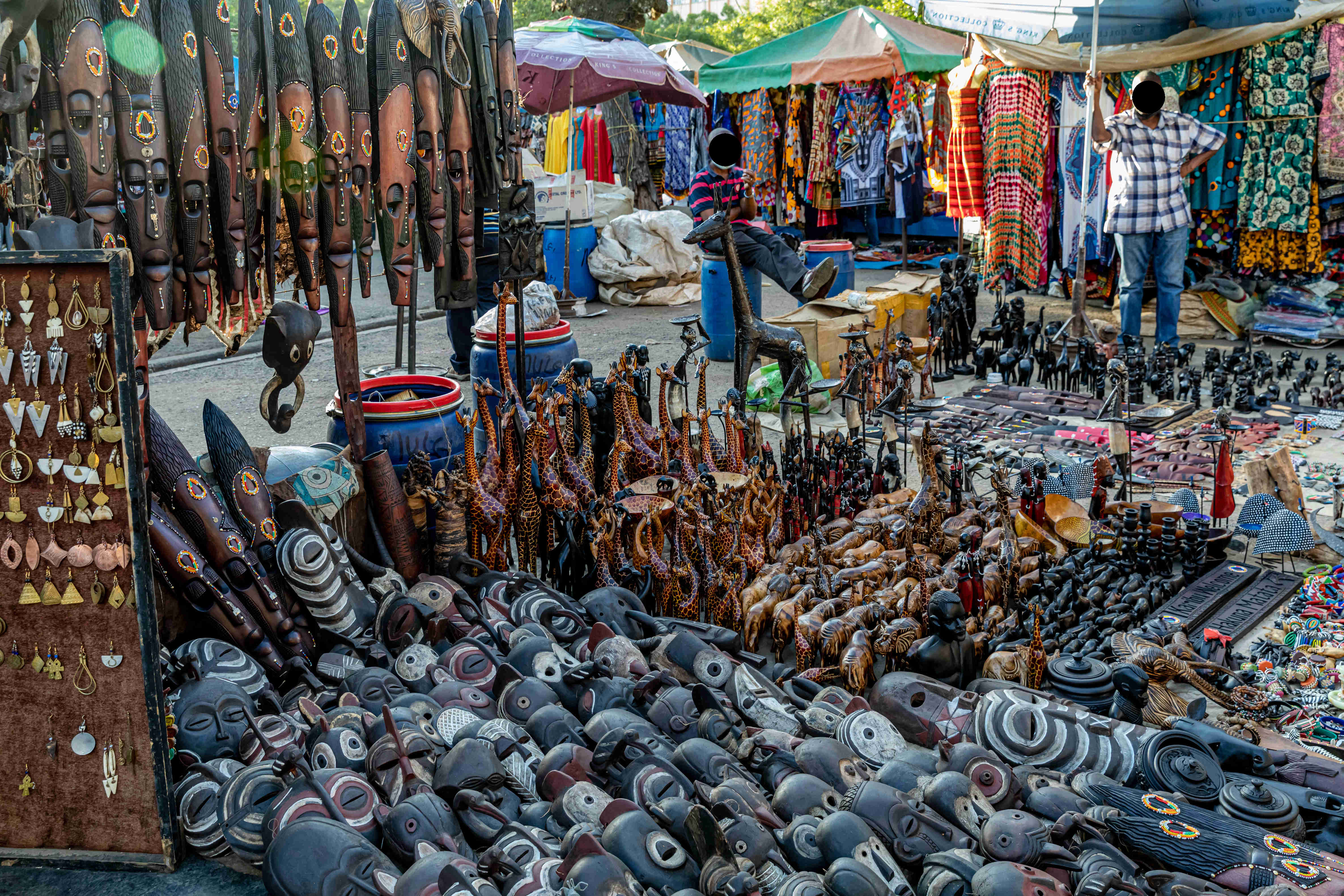}}; \\
            \tikz \node[inner sep=1pt, fill=none, text=black] {\small YFCC-100M \cite{YFCC}}; \\
            \tikz \node[line width=2mm, draw=OSVCOLOR, inner sep=0pt, text=black]{\includegraphics[width=.18\linewidth, height=.08\linewidth]{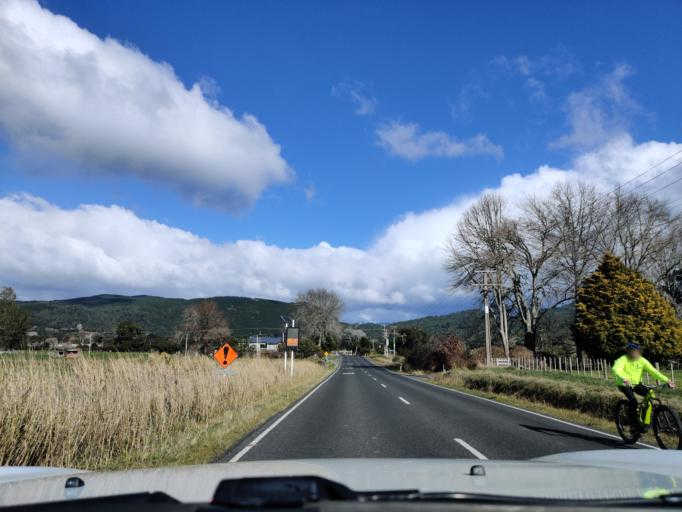}}; \\
             \tikz \node[inner sep=1pt, fill=none] {\small OSV-5M \cite{astruc2024openstreetview}};\\[0.0em]
             \end{tabular}
    }
\end{tabular}

%% file: text/1_introduction.tex
\epigraph{
    “The world has shrunk; today, we travel it at ten times the speed of a hundred years past.”
}{— \emph{Around the World in 80 Days}, {Jules Verne}}

Knowing where an image was captured is crucial for numerous applications, and yet most images lack geolocation metadata~\cite{flatow2015accuracy}. In archaeology and cultural heritage, location data help catalog and interpret historical artifacts~\cite{daoud2013mining,smith2001disambiguating}, enabling better preservation and contextual understanding. In fields like forensics and investigative journalism, recovering intentionally removed GPS data can have significant implications~\cite{bamigbade2024computer,yokota2020revisited}, such as verifying the authenticity of news images and reconstructing crime scenes or missing persons' last known locations. Moreover, geolocation helps organizing multimedia archives for efficient retrieval~\cite{nikolaidou2022survey,delozier2016creating}. These applications motivate the long-standing computer vision challenge of global visual geolocation: inferring the location of an image purely from its visual content~\cite{hays2008img2gps,Vp2017revisiting}.

\paragraph{Modeling Spatial Ambiguity.} As illustrated in \cref{fig:teaser}, the precision with which images can be localized---their \emph{localizability}~\cite{astruc2024openstreetview,izbicki2020exploiting}---varies significantly. A featureless beach could have been photographed almost anywhere, while a landmark like the Eiffel Tower can be pinpointed with meter-level accuracy. In intermediate cases, such as a close-up photo of a kangaroo, the location can be narrowed down to Australia but specifying its exact spot is challenging. This inherent ambiguity should be reflected in geolocation methods and metrics. However, most existing approaches produce deterministic predictions using regression~\cite{astruc2024openstreetview,haas2023pigeon}, classification~\cite{weyand2016planet,clark2023where,theiner2022interpretable}, or retrieval-based techniques~\cite{oliva2006building,martin2001database,Vp2017revisiting}, thus disregarding the varying localizability of images.

Modeling spatial ambiguity in computer vision tasks, such as object localization, has improved their robustness and interpretability~\cite{merrill2022symmetry,deng2022deep,xu20246d}. Furthermore, generative models like diffusion~\cite{ho2020denoising,song2020score} and flow matching~\cite{lipman23flow} have been successfully applied to complex tasks with noisy supervision, including image~\cite{ho2022cascaded}, video~\cite{blattmann2023stable}, speech~\cite{popov2021grad}, and music~\cite{mittal2021symbolic} generation. Inspired by these advances, we propose to bridge the gap between traditional geolocation and modern generative methods.

\paragraph{Generative Geolocation.} In this work, we present a novel generative approach to global visual geolocation by using diffusion or flow-matching to denoise random locations into accurate estimates conditioned on image features.
We extend recent manifold-based flow matching techniques~\cite{chen2024riemannian} such that the denoising operate directly on geographic coordinates. This allows our model to take into account the Earth's spherical geometry when learning the relationship between the content of images and their location. Additionally, we extend recent developments in density estimation for flow matching~\cite{lipman23flow} to our setting, enabling our models to compute the likelihood of any location given an image and provide a quantifiable estimate of its localizability.

Our approach achieves higher accuracy than state-of-the-art geolocation methods on three standard large-scale datasets: OpenStreetView-5M~\cite{astruc2024openstreetview}, iNat21~\cite{van2021benchmarking}, and YFCC-100M~\cite{YFCC}. Moreover, we introduce the task of probabilistic visual geolocation, where the model predicts a probability distribution over all possible locations rather than a single point. We implement new metrics and baselines for this task, demonstrating the advantages of our generative approach in capturing ambiguous yet informative visual cues.
Our contributions are as follows: 
\begin{compactitem} 
\item We introduce the first application of diffusion and Riemannian flow matching methods for visual geolocation by directly denoising spatial coordinates, using manifold-based methods to respect the Earth's spherical geometry.  
\item We extend recent density estimation methods to our geolocalization setting, thus modeling the conditional distribution over locations and quantifying localizability.
\item We demonstrate that modeling the ambiguity in geolocation leads to improved performance, achieving state-of-the-art results on three public datasets. 
\item We propose the task of probabilistic visual geolocation, along with associated metrics and baselines. 
\end{compactitem}

%% file: text/2_related.tex
\paragraph{Global Visual Geolocation.}
Visual geolocation consists in predicting an image's geographic coordinates, focusing on large-scale and generalizability to unseen areas~\cite{hays2015large}. Existing methods are categorized into image retrieval-based, classification-based, and hybrid approaches. Retrieval-based methods locate an image by finding the most similar one in a database using handcrafted~\cite{hays2008img2gps,oliva2006building,martin2001database} or deep features~\cite{Vp2017revisiting}, but they require dense databases and may struggle in sparse or dynamic environments. Classification-based methods partition the globe into discrete cells, such as regular grids~\cite{weyand2016planet}, adaptive cells~\cite{clark2023where}, semantic regions~\cite{theiner2022interpretable}, or administrative boundaries~\cite{pramanick2022where,haas2023pigeon} and treat geolocation as a classification task. Hybrid approaches combine classification with regression \cite{astruc2024openstreetview} or retrieval to mitigate discretization issues, employing contrastive losses~\cite{Vp2017revisiting,kordopatis2021leveraging} or prototype networks~\cite{haas2023pigeon}.
Izbicki \etal ~\cite{izbicki2020exploiting} propose a model that predicts a distribution of probability anywhere on Earth, but only evaluates its performance in terms of geolocation performance.

\paragraph{Uncertainty-Aware Localization.} 
Estimating uncertainty in neural networks is a long-standing problem in computer vision~\cite{kendall2017uncertainties}. This is particularly important for fine-grained localization tasks, especially in robotic applications~\cite{dellaert1999monte,deng2022deep,levinson2010robust}. In 6DOF or human body pose estimation~\cite{merrill2022symmetry}, uncertainty is often modeled by predicting localization heatmaps~\cite{tompson2014joint,pavlakos2017coarse}. This challenge is typically addressed using Bayesian statistics~\cite{mullane2011random} and variational inference~\cite{zangeneh2023probabilistic}, which have been adapted to deep learning models~\cite{kendall2016modelling}.

Generative approaches, such as diffusion models~\cite{berryshedding} and normalizing flows~\cite{grathwohlffjord}, have shown promise in explicating uncertainty. These methods have been applied for uncertainty estimation for tasks such as image segmentation~\cite{wolleb2022diffusion}, source localization \cite{huang2023two}, and LiDAR localization \cite{li2024diffloc}.

\paragraph{Generative Models.}
Diffusion models have emerged as a transformative force in generative modeling~\cite{ho2020denoising, song2020score, song21iclr}, demonstrating remarkable success across diverse applications including image synthesis~\cite{Rombach_2022_CVPR, saharia2022photorealistic}, video generation~\cite{ho2022imagen, polyak2024movie}, and human-centric tasks~\cite{courant2025exceptional, petrovich24stmc}. Flow matching models~\cite{lipman2022flow} have further advanced the field by offering a simplified training objective. Recent research has also explored learning directly on data distribution manifolds~\cite{chen2023riemannian}. Generative models show particular robustness in handling data with irreducible uncertainty~\cite{dufour2022scam, dufour2024dont,Mackowiak2021CVPR}. While these models have been adapted for discriminative tasks~\cite{li2023your}, bridging the performance gap with traditional discriminative models remains an active research challenge. In our work, we demonstrate that generative models can effectively tackle the geolocation task by learning the manifold of the underlying data distribution, ultimately achieving superior performance compared to discriminative approaches.

%% file: text/3_method.tex
We first present our diffusion-based approach (\cref{sec:diffusion}) and extend it to the Riemannian flow matching framework (\cref{sec:flow}), see \cref{fig:training} for a visual summary of the difference between these techniques. We then describe how to predict location distribution  (\cref{sec:pred}). Finally, we detail implementation choices in \cref{sec:implem}. %

\paragraph{Notations.} 
Given an image $c$, we aim to predict the most likely location $x_0$ where it was taken. More broadly, we model the conditional probability distribution $p(y \mid c)$, where $y$ can be any point on Earth, modeled as the unit sphere $\cS_2$ in $\bR^3$. %
Throughout the paper, we will denote %
pure random noise as $\epsilon$, the noisy coordinates as $x_t$ for a timestep $t$, and the network to optimize as $\psi$.
\subsection{Geographic Diffusion}
\label{sec:diffusion}
 In this section, we describe our diffusion-based generative approach to image geolocation.
Traditional diffusion models progressively add Gaussian noise to data and train a neural network to reverse this noising process~\cite{ho2020denoising, song2020score}. Once trained, the model can generate new data samples by starting from pure noise and performing iterative denoising. %

In our setting, we operate in the Euclidean space $\mathbb{R}^3$. Given a coordinate-image pair $(x_0, c)$ from a dataset $\Omega$ of geotagged images, we add noise to the true coordinates $x_0$ and train a neural network $\psi$ to predict this noise conditioned on the image $c$, thus learning the relationship between visual content and geographic locations. We can then predict the location of an unseen image by iteratively denoising a random initial coordinate $\epsilon$.

\paragraph{Training.}
We sample a coordinate-image pair $(x_0,c)$ from $\Omega$, and random coordinates $\epsilon$ from $\mathcal{N}(0,\mathbf{I})$, where $\mathbf{I}$ the identity matrix in $\bR^3$. We randomly select a time variable $t \in [0,1]$ representing the diffusion time step and use a scheduling function $\kappa(t): [0, 1] \to [0, 1]$ with $\kappa(0)=0$ and $\kappa(1)=1$ to control the noise level added to the coordinates. The noisy coordinates $x_t$ are defined as
\begin{align}\label{eq:noise_diff}
    x_t = \sqrt{1-\kappa(t)} x_0 + \sqrt{\kappa(t)} \epsilon~.
\end{align}
Our network $\psi$ takes as input the noisy coordinate $x_t$, the noise level $\kappa(t)$, and the image embedding $c$, and is tasked with predicting the corresponding pure noise $\epsilon$. For ease of notation, we will omit the conditional dependence of $\psi$ on $\kappa(t)$ in the rest of the paper. The model is trained to minimize the diffusion loss function:
\begin{equation}\label{eq:loss_diff}
\mathcal{L}_{\text{D}} =
\mathbb{E}_{x_0, c, \epsilon, t}
\left[ \left\Vert \psi(x_t \mid c) - \epsilon \right\Vert^2 \right]~, 
\end{equation}
where the expectation is over $(x_0, c) \sim \Omega$, $\epsilon \sim \mathcal{N}(0, \mathbf{I})$, and $t \sim \cU[0,1]$, the uniform distribution over $[0,1]$.

\begin{figure}[t]
    \centering
    \input{figures/training}
    \caption{{\bf Generative Framework.} We implement three generative approaches for geolocation: diffusion in $\mathbb{R}^3$, flow matching in $\mathbb{R}^3$, and Riemannian flow matching directly on $\mathcal{S}_2$. This figure provides the formulas for the noising processes and the loss functions for each approach.}
    \label{fig:training}
\end{figure}

\begin{figure*}[t]
    \centering
    \input{figures/pipeline}
    \caption{{\bf Inference Pipeline.}  
    We start by embedding the image to be localized into a vector using a frozen image encoder. We then sample a random noise $\epsilon$ in $\mathbb{R}^3$ or on $\mathcal{S}_2$, projected here onto the sphere. We iteratively remove the noise using either the reverse diffusion or flow-matching equations for $t=1$ to $0$. The final point of this trajectory is our predicted location. Additionally, our model be queried to predict a probability distribution at any point on the sphere by solving an Ordinary Differential Equation (ODE) system.   %
    }
    \label{fig:pipeline}
\end{figure*}
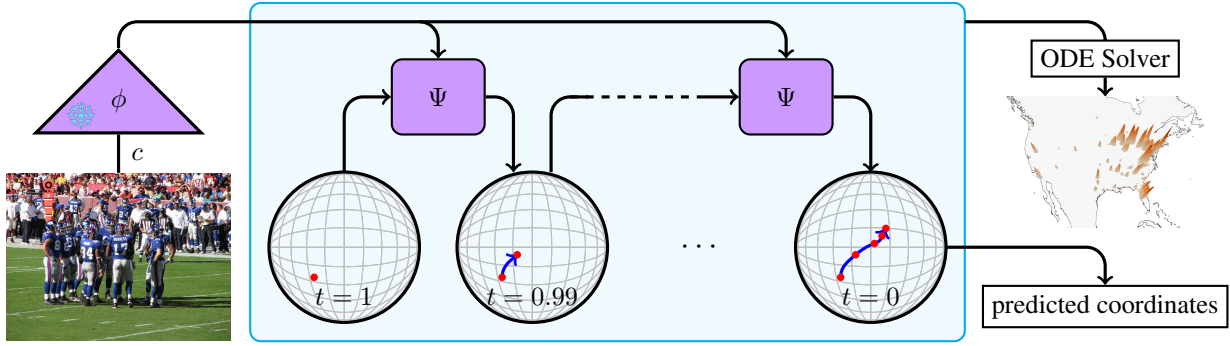

\paragraph{Inference.} To predict the likely locations for a new image $c$, we start by sampling a random coordinate $\epsilon \sim \mathcal{N}(0, \mathbf{I})$ and initialize $x_1=\epsilon$. We then iteratively refine the coordinate $x_t$ over $N$ timesteps from $t = 1$ to $t = 0$ using the Denoising Diffusion Implicit Models (DDIM) sampling procedure~\cite{song2020denoising}. The update equations are
\begin{align}\label{eq:ode_diff}
    x_{t - dt} &= \sqrt{1-\kappa(t)}\hat{x}_{t} + \sqrt{\kappa(t)}\psi(x_t |c)~,\\
    \hat{x}_t &= \frac{1}{\sqrt{1-\kappa(t)}}\left(x_t - \sqrt{\kappa(t)}\psi(x_t|c)\right)~,
\end{align}
where $dt$ is the time step size, and $\hat{x}_t$ is the estimate of the denoised coordinate at time $t$. At the end of the denoising process ($t=0$), we project $\hat{x}_0$ to $\cS_2$ to ensure that it is a valid location on the Earth's surface. See \cref{fig:pipeline} for an illustration of the inference process.

\subsection{Extension to Riemannian Flow Matching}
\label{sec:flow}
Flow matching generalizes diffusion models with increased performance and versatility~\cite{lipman23flow}. We extend our approach to this setting, and leverage Riemannian flow matching to directly work on the sphere $\mathcal{S}^2$. In each setting, we still denote our network $\psi$ but redefine an alternative noising process (\cref{eq:noise_diff}), loss function (\cref{eq:loss_diff}), and denoising procedure (\cref{eq:ode_diff}).

\paragraph{Flow Matching in $\bR^3$.} 
In flow matching, we define a mapping from the true coordinates $x_0$ to random noise $\epsilon$:
\begin{align}
x_t = (1-\kappa(t))x_0 +   \kappa(t) \epsilon~.
\end{align}
This defines the  following velocity field:
\begin{align}
    v(x_t) =  \frac{d x_t}{dt}  = \dot{\kappa}(t)(\epsilon-x_0)~,
\end{align}
where $\dot{\kappa}$ the derivative of $\kappa$ with respect to $t$.
We train our model $\psi$ to predict this velocity field conditionally to the image $c$:
\begin{align}
     \mathcal{L}_{\text{FM}} = 
    \mathbb{E}_{x_0,c,\epsilon,t}
    \left[ \left\| \psi(x_t \mid c) - v(x_t)
    \right\|^2\right]~,
\end{align}
with the expectation taken over the same distributions as in \cref{eq:noise_diff}.
During inference, we solve the Ordinary Differential Equation (ODE) initialized at a random coordinate $\epsilon$, integrating backward from $t=1$ to $t=0$ using the predicted velocity field $\psi(x_t \mid c)$:
\begin{align}\label{eq:ode_nfm}
    x_{t-dt} = x_t - \psi(x_t|c)dt~.
\end{align}
At the end of the integration, we project $x_0$ onto the sphere.

\paragraph{Riemannian Flow Matching on the Sphere.} 
Since our data lies on the sphere $\mathcal{S}^2$, it is natural to constrain the flow matching process to this manifold. The Riemannian flow matching approach~\cite{chen2024riemannian} extends flow matching to Riemannian manifolds and requires three conditions: 
(i) all true coordinates $x_0$ lie on $\mathcal{S}^2$,
(ii) the noise samples $\epsilon$ lie on $\mathcal{S}^2$,
and 
(iii) the noisy coordinates $x_t$ remain on $\mathcal{S}^2$.

Condition (i) is naturally satisfied since we are working with coordinates on the Earth's surface. For condition (ii), we sample $\epsilon$ uniformly at random on $\mathcal{S}^2$. Unlike diffusion models, flow matching does not require the noise distribution to be Gaussian.
For condition (iii), we define the noisy coordinates along the geodesic between the true coordinate $x_0$ and the noise sample $\epsilon$, parameterized by $\kappa(t)$:
\begin{align}
x_t=\exp_{{x_0}} \left(\kappa(t) \log_{{x_0}}(\epsilon) \right)~,
\end{align}
where $\log_{x_0}$ is the logarithmic map mapping point of $\cS_2$ to the tangent space at $x_0$, and $\exp_{x_0}$ is the exponential map, mapping tangent vectors back to the manifold (see appendix for detailed expressions).
This parametrization induces a velocity field $v(x_t)$ defined on the tangent space of $x_t$:
\begin{align}
v(x_t)=\dot{\kappa}(t) \cdot D(x_t)~,
\end{align}
where $D(x_t)$ is the tangent vector at $x_t$ pointing along the geodesic from $x_0$ to $\epsilon$, with magnitude equal to the geodesic distance between $x_0$ and $\epsilon$.
We train our model $\psi$ to approximate this velocity field by minimizing
\begin{align}
    \mathcal{L}_{\text{RFM}} = 
    \mathbb{E}_{x_0,c,\epsilon,t}
    \left[
     \left\|\psi(x_t | c) - v(x_t)\right\|_{x_t}^2
    \right]~,
\end{align}
with $(x_0, c) \sim \Omega$, $\epsilon \sim \cU(\cS_2)$ $t \sim \cU[0,1]$,  and $\|\cdot\|_{x_t}$  denotes the norm induced by the Riemannian metric on the tangent space at $x_t$.
During inference, we solve the ODE starting from a random point $\epsilon \in \mathcal{S}^2$ and integrating backward from $t=1$ to $t=0$ using the predicted velocity and projecting the iterates on the manifold at each step:
\begin{align}\label{eq:ode_rnfm}
    x_{t-dt} = \exp_{{x_t}}\left(-dt  \psi(x_t\mid c)\right) ~.
\end{align}
This ensures that the trajectory remains on the sphere $\mathcal{S}^2$ throughout the integration process.

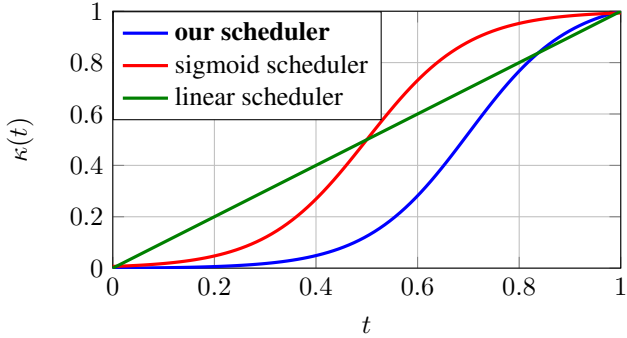
\begin{figure}[t]
\input{figures/scheduler}
\vspace{-8mm}
\caption{{\bf Scheduler.} We chose a noise scheduler that assigns more weights to the beginning of the diffusion process.}
\label{fig:scheduler}
\end{figure}
\subsection{Guidance and Density Prediction}
\label{sec:pred}
We can incorporate guidance to our models' to improve their accuracy, and compute the spatial distribution of locations $p(y \mid c)$ for an image $c$.

\paragraph{Guided Geolocation.}
We adapt the idea of classifier-free guidance~\cite{ho2021classifier} to our setting. We train the network $\psi$ to learn both the conditional distribution $p(y \mid c)$ and the unconditional distribution $p(y \mid \varnothing)$ by randomly dropping the conditioning on the image $c$ for a fraction of the training samples (e.g., 10\%). During inference, we replace $\psi$ in the ODE of \cref{eq:ode_diff,eq:ode_nfm,eq:ode_rnfm} by $\hat{\psi}$ defined as follows:
\begin{align}
  \hat{\psi}(x_t \mid c) = {\psi}(x_t \mid c) + \omega 
  \left( {\psi}(x_t \mid c) - {\psi}(x_t \mid \varnothing)
  \right)~,
\end{align}
where $\psi(x_t, \varnothing)$ is the prediction without conditioning, and $\omega \geq 0$ is the guidance scale. A guidance scale of $\omega = 0$ corresponds to the standard approach, while higher values of $\omega$ place more emphasis on the conditioning, leading to sharper distributions.
Note that changing the guidance scale does not require to retrain the model.

\paragraph{Predicting Distributions.} After training $\psi$, we can compute the likelihood $p(y \mid c)$ of any coordinate $y$ corresponding to the image $c$. We provide here the derivation in the Euclidean flow matching setting, where it is the most straightforward. Our derivations are inspired by \cite[Appendix C]{lipman23flow} and rely on the logarithmic mass conservation theorem~\cite{ben2022matching, villani2009optimal}. We provide more details in the appendix.
\begin{prop}
  Given a location $y \in \mathcal{S}^2$ and an image $c$, consider solving the following ordinary differential equation system for $t$ from $0$ to $1$:
    \begin{align}
    \frac{d}{dt}
     \begin{bmatrix}
     x_t \\
     f(t)
     \end{bmatrix}
     =
     \begin{bmatrix}
     \psi(x(t) \mid c) \\
     -\divergence\,\psi(x_t \mid c)
     \end{bmatrix}
     \;\text{with}\;
    \begin{bmatrix}
     x_0 \\
     f(0)
     \end{bmatrix}
     =
     \begin{bmatrix}
     y \\
     0
     \end{bmatrix}~,
    \end{align}
Then the log-probability density of $y$ given $c$ is:
$\log p(y \mid c) = \log p_\epsilon(x(1) \mid c) - f(1)$ where $p_\epsilon$ is the known distribution of the pure noise $\epsilon$, and $f(t)$ accumulates the negative divergence of the velocity field along the trajectory $x_t$.
\end{prop}

We solve this system numerically using the fifth-order Dormand-Prince-Shampine variant of the Runge-Kutta stepping scheme \cite{butcher2015runge, dormand1980family}, as implemented in TorchDiffEq \cite{torchdiffeq}.

\subsection{Implementation}
\label{sec:implem}
We detail here our choice of scheduler and model architecture, which are shared across all implementations.

\paragraph{Scheduler.} We observed better results with schedulers $\kappa(t)$ that assign more time to the beginning of the noising process \ie when the coordinates remain close to the true location. Our intuition is that this encourages the network to focus on learning fine-grained location cues in images rather than the easier, continent-level information. As illustrated in \cref{fig:scheduler}, we set $\kappa(t)$ as a skewed sigmoid function:
\begin{align} \label{eq:scheduler}
\kappa(t) = \frac{\sigma(\alpha) - \sigma(\alpha + t (\beta- \alpha))}{\sigma(\alpha) - \sigma(\beta)}~, \end{align}
where $\sigma(t) = 1/(1+\exp(-t))$ is the sigmoid function, and $\alpha,\beta$ control the skewness of the sigmoid. In practice, we use $\alpha=-3$ and $\beta=7$. 

\paragraph{Model Architecture.}
The network $\psi$ used for all methods is composed of $6$ residual blocks which take as input%
 the current noisy coordinate $x_t$, the embedding of image $c$, and the current noise level $\kappa(t)$. The image $c$ is embedded using a pre-trained and frozen image encoder $\phi$ into a $d$-dimensional vector. Additionally, we compute $d$-dimensional Fourier features of $\kappa(t)$ to capture fine-grained temporal information.

Each block of $\psi$ follows a similar architecture to the DiT model~\cite{peebles2023scalable}, consisting of a Multi-Layer Perceptron (MLP) with GELU activations~\cite{hendrycks2016gaussian}. We modulate the coordinate embeddings according to the conditioning using adaptive layer normalization (AdaLN). The network concludes with an AdaLN layer and a linear layer that outputs the predicted noise. See the appendix for more details.

\begin{table}[t]
    \centering
        \caption{{\bf Geolocation Performance.} We compare the geolocation precision of traditional and generative visual geolocation methods, and three implementation of our generative approaches.}
        \vspace{-1mm}
    \input{tables/geoloc}
    \label{tab:geoloc}
\end{table}

%% file: figures/training.tex
\def\scale{1.5}
\footnotesize{
\begin{tabular}{c@{\,}c}
\begin{tabular}{r@{\;}l}
\multicolumn{2}{c}{
\begin{tikzpicture}
    \earth{0}{0}{\scale}
    
    \node[circle, draw=none, scale=0.5, fill=green] (xhat) at (-0.5*\scale,-0.5*\scale) {};
    \node[draw=none,right= 0mm of xhat] {${x_0}$};

    \node[circle, draw=none, scale=0.5, fill=red] (eps) at (+0.5*\scale,+0.5*\scale) {};
    \node[draw=none, above =0mm of eps] {$\epsilon$};

    \draw [thick] (xhat) to[out=90, in=180] node[midway, circle, draw=none, scale=0.5, fill=orange] (xt) {} (eps);
    \node[draw=none, left=0mm of xt] {$x_t$};

    \node[circle, draw=none, scale=0.5, fill=blue] (x1) at (+0.1*\scale,-0.3*\scale) {};
    \node[draw=none, above right=0mm and 0mm of x1, fill=white, inner sep=0pt] {$\Psi(x_t\mid c)$};

    \draw [thick] (xhat) to[out=90, in=180] node[midway, circle, draw=none, scale=0.5, fill=orange] (xt) {} (eps);
    \node[draw=none, left=0mm of xt] {$x_t$};

    \draw [->, thick] (xt) to node[draw=none, fill=white, inner sep=0, pos=0.5, xshift=-3pt, yshift=4pt, anchor=east]{$v(x_t)$} (0.2,1);

    \draw [->, thick] (xt) to[out=-30,in=90] (x1);

\end{tikzpicture}
}\\
\tikz \node[circle, draw=none, scale=0.5, fill=green] (xhat) at (0,0) {};
&
\footnotesize $
{x_0}$: true location
\\
\tikz \node[circle, draw=none, scale=0.5, fill=red] (xhat) at (0,0) {};
&\footnotesize $\epsilon$: sampled noise
\\
\tikz \node[circle, draw=none, scale=0.5, fill=orange] (xt) at (0,0) {};
&\footnotesize $x_t$: noisy location
\\
\tikz \node[circle, draw=none, scale=0.5, fill=blue] (xhat) at (0,0) {};
&\footnotesize $\psi(x_t\mid c)$: prediction\\
\tikz[baseline=-0.5ex] \draw[->] (0,0) -- (0.2,0);
&\footnotesize $v(x_t)$: velocity field
\end{tabular}
&
\raisebox{+0cm}{
\footnotesize{
\begin{tabular}{c}
\toprule
Diffusion \\
\greyrule
\footnotesize $x_t=\sqrt{1-\kappa(t)} {x_0} + \sqrt{\kappa(t)} \epsilon$ \\
\footnotesize $\cL_{\text{D}}=\left\|\psi(x_t\mid c)-\epsilon\right\|^2$\\
\midrule
Flow Matching \\
\greyrule
\footnotesize $x_t=(1-\kappa(t)) {x_0} + \kappa(t) \epsilon$ \\
\footnotesize $\cL_{\text{FM}}=\left\|\psi(x_t\mid c)-v(x_t)\right\|^2$\\
\midrule
Riemannian Flow Matching \\\greyrule
\footnotesize $x_t=\exp_{{x_0}} \left(\kappa(t) \log_{{x_0}}(\epsilon) \right)$ \\
\footnotesize $\cL_{\text{RFM}}=\left\|\psi(x_t\mid c)-v(x_t)\right\|^2_{x_t}$\\\bottomrule
\rule{0pt}{4mm} $\kappa(t)$: noise scheduler
\end{tabular}
}}
\end{tabular}
}

%% file: figures/pipeline.tex
\def\xinput{-3.0}
\def\xencoder{-3.0}
\def\xearthone{0}
\def\xdotsone{0}
\def\xearthtwo{2.5}
\def\xdotstwo{4.75}
\def\xearththree{7}
\def\earthsize{1}

\makeatletter
\@ifundefined{xheatmap}{
    \def\xheatmap{8.8}
}
\makeatother

\def\yearth{0}
\def\ynet{2}
\def\yencod{1.5}
\def\netheight{1}
\def\ytimes{-0.65}

\def\deltaonex{-0.4}
\def\deltatwox{-0.2}
\def\deltathreex{-0.05}
\def\deltafourx{+0.15}
\def\deltafivex{+0.20}

\def\deltaoney{-0.4}
\def\deltatwoy{-0.1}
\def\deltathreey{+0.05}
\def\deltafoury{+0.15}
\def\deltafivey{+0.25}

\definecolor{NETWORKCOLOR}{RGB}{204,153,255}
\pgfmathsetmacro{\xnetworkone}{1*\xearthone/2 + \xearthtwo/2}
\pgfmathsetmacro{\xnetworktwo}{(\xearthtwo/4 + 3*\xearththree/4}

\begin{tikzpicture}

\tikzstyle{net}=[thick, text width=1.0cm, minimum height=\netheight cm, align=center, draw=black, rounded corners, anchor=center, fill=NETWORKCOLOR]

 \draw[rounded corners, fill=cyan!5, draw = cyan, thick] (\xearthone-\earthsize-0.25, \yearth-\earthsize-0.25) rectangle (\xearththree+\earthsize+0.25, \ynet+\netheight+0.25) {};

    \node[draw=none, anchor=south, inner sep=0pt] (imginput) at  (\xinput,\yearth-\earthsize-0.25) {\includegraphics[width=.17\textwidth]{images/yfcc_med.jpg}};

    \node [draw=none] at ($(imginput.north) + (0.25,0.25)$) {$c$};

     \node[isosceles triangle, isosceles triangle apex angle=90, minimum width=1.2cm,minimum height=1.1cm, draw, very thick, anchor=west, fill=NETWORKCOLOR, rotate=+90] (imgencoder) at (\xencoder, \yencod) {};

     \node [draw=none] at (imgencoder.center) {$\phi$};

      \node [draw=none] at (\xencoder-0.5, \yencod+0.25) {\includegraphics[width=0.35cm]{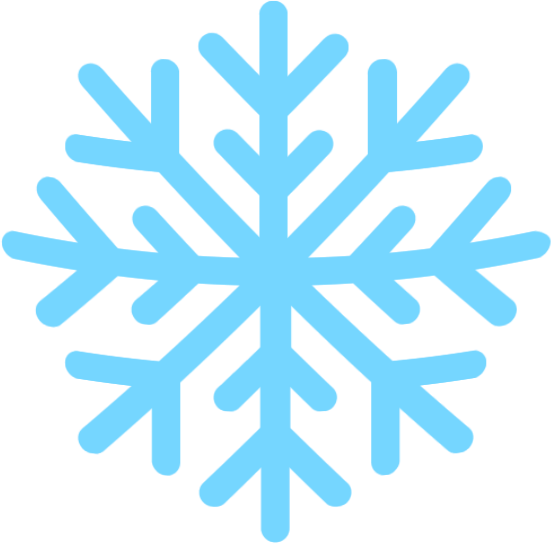}};

      \earth{\xearthone}{\yearth}{\earthsize}
      \earth{\xearthtwo}{\yearth}{\earthsize}
      \node[draw=none] at (\xdotstwo,\yearth) {\large $\cdots$};
      \earth{\xearththree}{\yearth}{\earthsize}

      \node [draw=none] at (\xearthone, \ytimes) {$t=1$};
      \node [draw=none] at (\xearthtwo, \ytimes) {$t=0.99$};
      \node [draw=none] at (\xearththree, \ytimes) {$t=0$};

      \fill[red] (\xearthone+\deltaonex,\yearth+\deltaoney) circle (0.05cm);

       \draw [very thick, blue, ->] (\xearthtwo+\deltaonex,\yearth+\deltaoney) to [out=90, in=220] (\xearthtwo+\deltatwox,\yearth+\deltatwoy);

      \fill[red] (\xearthtwo+\deltaonex,\yearth+\deltaoney) circle (0.05cm);
      \fill[red] (\xearthtwo+\deltatwox,\yearth+\deltatwoy) circle (0.05cm);

      \draw [very thick, blue, ->] (\xearththree+\deltaonex,\yearth+\deltaoney) to [out=90, in=220] (\xearththree+\deltatwox,\yearth+\deltatwoy) to [out=40, in=200] (\xearththree-\deltathreex,\yearth+\deltathreey)
      to [out=20, in=260] (\xearththree+\deltafourx,\yearth+\deltafoury)
      -- (\xearththree+\deltafivex,\yearth+\deltafivey);

       \fill[red] (\xearththree+\deltaonex,\yearth+\deltaoney) circle (0.05cm);
      \fill[red] (\xearththree+\deltatwox,\yearth+\deltatwoy) circle (0.05cm);
      \fill[red] (\xearththree-\deltathreex,\yearth+\deltathreey) circle (0.05cm);
      \fill[red] (\xearththree+\deltafourx,\yearth+\deltafoury) circle (0.05cm);
      \fill[red] (\xearththree+\deltafivex,\yearth+\deltafivey) circle (0.05cm);

     \node[net] (networkone) at (\xnetworkone, \ynet) {$\Psi$};
     \node[net] (networktwo) at (\xnetworktwo , \ynet) {$\Psi$};

       \node[draw=none, anchor=south west, inner sep=0pt] (heatmap) at  (\xheatmap,\yearth+\earthsize-0.75) {\includegraphics[width=.15\textwidth, height=.10\textwidth]{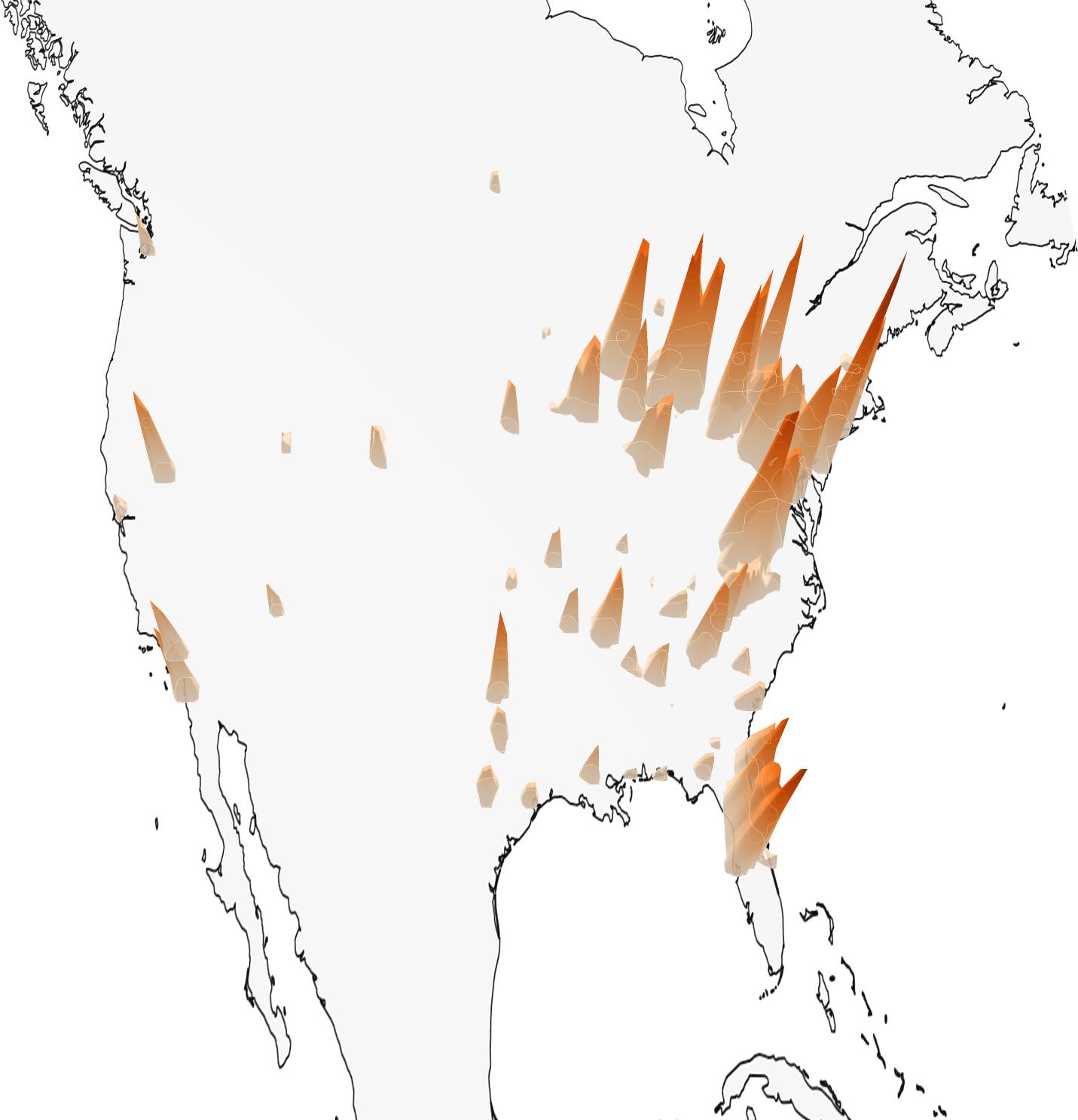}};

       \node[thick, draw=black] (ode)  at ($(heatmap.north)+(0,0.5)$) {ODE Solver};

     \node[thick, draw=black, anchor=north] (predicted)  at ($(heatmap.south)+(0,-0.75)$) {predicted coordinates};  %

     \coordinate (globe) at (\xearththree,\yearth);

     \draw [very thick, ->] (\xearththree+\earthsize,\yearth)  -- ($(predicted.north |- globe) + (-0.25,0)$) to[out=0, in=90]  ++ (0.25,-0.25) -- (predicted.north);
      
    \draw [very thick, ->] (\xearthone, \yearth+\earthsize) to (0,\ynet-0.25) to[out=90, in=180]  ++ (0.25,0.25) -- (networkone.west);
    \draw [very thick, ->] (networkone.east)   -- (\xearthtwo-0.5,\ynet) to[out=0, in=+90]  ++ (+0.25,-0.25) -- (\xearthtwo-0.25, \yearth+\earthsize);
     \draw [very thick, -] (\xearthtwo+0.25, \yearth+\earthsize) to (\xearthtwo+0.25, \ynet-0.25)
     to[out=90, in=180]  ++ (0.25,0.25) -- ++ (0.25,0);
     \draw [very thick,  dashed] (\xearthtwo+0.5, \ynet) -- ($(networktwo.west)+(-0.5,0)$);
     \draw [very thick, ->] ($(networktwo.west)+(-0.5,0)$) -- (networktwo.west);
     \draw [very thick, ->] (networktwo.east)   -- (\xearththree-0.25,\ynet) to[out=0, in=+90]  ++ (+0.25,-0.25) -- (\xearththree, \yearth+\earthsize);

    \draw [very thick, ->] 
    (imginput) -- (imgencoder)
    -- (\xinput, \ynet + \netheight/2 + 0.25)
    to[out=90, in=180] ++(0.25, 0.25)
    -- ($(\xnetworkone-0.25, \ynet + \netheight/2 + 0.5)$)
    to[out=0, in=90] ++(0.25, -0.25)
    -- (networkone.north);

    \draw [very thick, ->] 
    ($(\xnetworkone - 0.25, \ynet + \netheight/2 + 0.5)$) --
    ($(\xnetworktwo - 0.25, \ynet + \netheight/2 + 0.5)$) to[out=0, in=90] ++(0.25, -0.25)
    -- (networktwo.north) ;

    \node [draw=none] (boxout) at ( \xearththree + \earthsize + 0.25 , \ynet + 1 ) {};
    \draw [very thick, ->] (boxout.center)
-- ($(ode.north|- boxout)+(-0.25,0)$) to[out=0, in=90] ++(0.25, -0.25) -- (ode) -- (heatmap);

\end{tikzpicture}

%% file: figures/scheduler.tex
    \begin{tikzpicture}
    \begin{axis}[
        xlabel={$t$},
        domain=0:1,
        samples=100,
        width=\linewidth,
        ylabel={$\kappa(t)$},
        height=5cm,
        grid=major,
        enlargelimits=0.0,
        legend style={at={(0,1)},anchor=north west, legend cell align=left}, 
        yticklabel style={/pgf/number format/fixed}
    ]
    \pgfmathsetmacro{\alpha}{-7}
    \pgfmathsetmacro{\beta}{3}

    \addplot[
        very thick,
        color=blue,
        samples=100
    ] 
    { (1/(1 + exp(-\alpha)) - 1/(1 + exp(-(\alpha + x * (\beta - \alpha))))) /
      (1/(1 + exp(-\alpha)) - 1/(1 + exp(-\beta))) };

      \addlegendentry{\bf our scheduler}

    \addplot[
        very thick,
        color=red,
        samples=100
    ] 
    {1/(1+exp(-10*(x-0.5))};
    \addlegendentry{sigmoid scheduler}

    \addplot[
        very thick,
        color=green!50!black,
        samples=100
    ] 
    { x };
    \addlegendentry{linear scheduler}

    \end{axis}
\end{tikzpicture}

%% file: tables/geoloc.tex
\hspace{-0.5cm}
\resizebox{1.05\linewidth}{!}{
\begin{tabular}{c} %

    \begin{tabular}{p{5.5mm}@{\;}l@{}c@{\,}c@{\,}c@{\,}c@{\,}c c} 
    \toprule
    && \multicolumn{5}{c}{OSV-5M \cite{astruc2024openstreetview}} & 
       \multicolumn{1}{c}{iNat21 \cite{van2021benchmarking}} \\
    \cmidrule(lr){3-7} \cmidrule(lr){8-8}
    && \multicolumn{1}{c}{geos. $\uparrow$} & \multicolumn{1}{c}{dist $\downarrow$} & \multicolumn{3}{c}{accuracy $\uparrow$ (in \%) }
    & dist $\downarrow$\\
    \cmidrule(lr){5-7}
    && \multicolumn{1}{c}{/5000 } & \multicolumn{1}{c}{(km)} & \multicolumn{1}{c}{country} & \multicolumn{1}{c}{region} & \multicolumn{1}{c}{city}
    & \multicolumn{1}{c}{ (km)} \\
    \midrule
\multirow{5}{*}{\rotatebox{90}{deterministic}}    %
&SC 0-shot \cite{haas2023learning}& 2273 & 2854 & 38.4 & 20.8 & \underline{14.8} & \\
&\greycell Regression \cite{astruc2024openstreetview}& \greycell 3028 & \greycell 1481 & \greycell 56.5 & \greycell 16.3 & \greycell \hphantom{1}0.7 & \greycell\\
&ISNs \cite{muller2018geolocation} & 3331 & 2308 & 66.8 & 39.4 & \hphantom{1}4.2 & \\
&\greycell Hybrid \cite{astruc2024openstreetview} & \greycell 3361 & \greycell 1814 & \greycell 68.0 & \greycell 39.4 & \greycell \hphantom{1}5.9 & \greycell \\
&SC Retrieval \cite{haas2023learning} & 3597 & 1386 & 73.4 & \bf 45.8 & \bf 19.9 & \\\greyrule
\multirow{7}{*}{\rotatebox{90}{generative}}
&\greycell Uniform & \greycell \hphantom{0}131 & \greycell 10052 & \greycell 2.4 & \greycell \hphantom{0}0.1 & \greycell 0.0 & \greycell 10,010\\
&vMF & 2776 & 2439 & 52.7 & 17.2 & \hphantom{1}0.6 & 6270 \\
&\greycell vMFMix \cite{izbicki2020exploiting} & \greycell 1746 & \greycell 5662 & \greycell 34.2 & \greycell 11.1 & \greycell \hphantom{1}0.3 & \greycell 4701\\
&\bf Diff $\bR^3$ (ours) & \underline{3762} & \underline{1123} & \underline{75.9} & 40.9 & \hphantom{1}3.6 & 3057\\
&\greycell \bf FM $\bR^3$ (ours) & \greycell 3688 & \greycell 1149 & \greycell 74.9 & \greycell 40.0 & \greycell \hphantom{1}4.2 & \greycell {2942}\\
&\bf RFM $\cS_2$ (ours) & \textbf{3767} & \textbf{1069} & \textbf{76.2} & \textbf{44.2} & \hphantom{1}5.4 & \textbf{2500}\\
\midrule
    \end{tabular} \\[6pt] %
    \begin{tabular}{p{5.5mm}@{}l@{} cc@{\;\;}c@{\;\;}c@{}c@{}c} 
    && \multicolumn{6}{c}{YFCC-4k \cite{YFCC,Vp2017revisiting}} \\
    \cmidrule(lr){3-8}
    && \multicolumn{1}{c}{geos. $\uparrow$} & \multicolumn{1}{c}{dist $\downarrow$} & \multicolumn{4}{c}{accuracy $\uparrow$ (in \%)} \\
    \cmidrule(lr){5-8}
    && /5000 & (km) & \multicolumn{1}{c}{25km} & \multicolumn{1}{c}{200km} & \multicolumn{1}{c}{750km} & \multicolumn{1}{c}{2500km} \\
    \midrule
    \multirow{6}{*}{\rotatebox{90}{deterministic}}
&PlaNet \cite{weyand2016planet} &&  & 14.3 & 22.2 & 36.4 & 55.8 \\
&\greycell CPlaNet \cite{seo2018cplanet}&\greycell&\greycell & \greycell 14.8 & \greycell 21.9 & \greycell 36.4 & \greycell 55.5\\
&ISNs \cite{muller2018geolocation}&& & 16.5 & 24.2 & 37.5 & 54.9 \\
&\greycell Translocator \cite{pramanick2022where} &\greycell&\greycell & \greycell 18.6 & \greycell 27.0 & \greycell 41.1 & \greycell 60.4\\
&GeoDecoder \cite{clark2023where}&& & \underline{24.4} & 33.9 & 50.0 & 68.7\\
&\greycell PIGEON \cite{haas2023pigeon}&\greycell &\greycell  & \greycell \underline{24.4} & \greycell \underline{40.6} & \greycell \bf {62.2} & \greycell \bf {77.7}\\\greyrule
\multirow{7}{*}{\rotatebox{90}{generative\quad\;}}
&Uniform & 131.2 & 10052  & 0.0 & 0.0 & 0.3 & 3.8 \\
&\greycell vMF & \greycell 1847 & \greycell 3563  & \greycell 4.8 & \greycell 15.0 & \greycell 30.9 & \greycell 53.4\\
&vMFMix \cite{izbicki2020exploiting} & 1356 & 4394  & 0.4 & 8.8 & 20.9 & 41.0 \\
&\greycell \bf Diff $\bR^3$ (ours) & \greycell {2845} & \greycell \underline{2461}  & \greycell 11.1 & \greycell {37.7} & \greycell {54.7} & \greycell 71.9 \\
&\bf FM $\bR^3$ (ours) & 2838 & 2514 & 22.1 & 35.0 & 53.2 & 73.1 \\
&\greycell \bf RFM $\cS_2$ (ours) & \greycell \underline{2889} & \greycell \underline{2461} & \greycell {23.7} & \greycell 36.4 & \greycell 54.5 & \greycell {73.6}\\
&\bf RFM$_\text{10M}$ $\cS_2$ (ours) & \bf 3210 & \bf 2058 & \bf 33.5 & \bf 45.3 & \underline{61.1} & \bf 77.7\\\bottomrule
    \end{tabular}
\end{tabular}
}

\ifx
   \end{tabular} \\[6pt] %
    \begin{tabular}{@{\!\!}p{5mm}@{}l@{} cc@{\;\;}c@{\;\;}c@{}c@{}c@{}c} 
    && \multicolumn{7}{c}{YFCC-4k \cite{YFCC,Vp2017revisiting}} \\
    \cmidrule(lr){3-9}
    && \multicolumn{1}{c}{geos. $\uparrow$} & \multicolumn{1}{c}{dist $\downarrow$} & \multicolumn{5}{c}{accuracy $\uparrow$ (in \%)} \\
    \cmidrule(lr){5-9}
    && /5000 & (km) &\multicolumn{1}{c}{1km} & \multicolumn{1}{c}{25km} & \multicolumn{1}{c}{200km} & \multicolumn{1}{c}{750km} & \multicolumn{1}{c}{2500km} \\
    \midrule
    \multirow{6}{*}{\rotatebox{90}{deterministic}}
&PlaNet \cite{weyand2016planet} && & 5.6 & 14.3 & 22.2 & 36.4 & 55.8 \\
&\greycell CPlaNet \cite{seo2018cplanet}&\greycell&\greycell& \greycell 7.9 & \greycell 14.8 & \greycell 21.9 & \greycell 36.4 & \greycell 55.5\\
&ISNs \cite{muller2018geolocation}&& & 6.7 & 16.5 & 24.2 & 37.5 & 54.9 \\
&\greycell Translocator \cite{pramanick2022where} &\greycell&\greycell& \greycell 8.4 & \greycell 18.6 & \greycell 27.0 & \greycell 41.1 & \greycell 60.4\\
&GeoDecoder \cite{clark2023where}&&& \underline{10.3} & \underline{24.4} & 33.9 & 50.0 & 68.7\\
&\greycell PIGEON \cite{haas2023pigeon}&\greycell &\greycell & \greycell \bf 10.4 & \greycell\underline{24.4} & \greycell \bf 40.6 & \greycell \bf 62.2 & \greycell \bf 77.7\\\greyrule
\multirow{6}{*}{\rotatebox{90}{generative\quad\;}}
&Uniform & 131.2 & 10052 & 0.0 & 0.0 & 0.0 & 0.3 & 3.8 \\
&\greycell vMF & \greycell 1847 & \greycell 3563 & \greycell 0.06 & \greycell 4.8 & \greycell 15.0 & \greycell 30.9 & \greycell 53.4\\
&vMFMix \cite{izbicki2020exploiting} & 1356 & 4394 & 0.00 & 0.4 & 8.8 & 20.9 & 41.0 \\
&\greycell \bf Diff $\bR^3$ (ours) & \greycell \underline{2845} & \greycell \textbf{2461} & \greycell 0.04 & \greycell 11.1 & \greycell \underline{37.7} & \greycell \underline{54.7} & \greycell 71.9 \\
&\bf FM $\bR^3$ (ours) & 2838 & 2514 & 2.37 & 22.1 & 35.0 & 53.2 & 73.1 \\
&\greycell \bf RFM $\cS_2$ (ours) & \greycell \bf 2889 & \greycell \bf 2461 & \greycell 3.38 & \greycell \underline{23.7} & \greycell 36.4 & \greycell 54.5 & \greycell \underline{73.6}\\\bottomrule
    \end{tabular}
    \fi

%% file: text/4_experiments.tex
We evaluate our models on two tasks: global visual geolocation and probabilistic visual geolocation. In the first task, the model predicts the most likely location where an image was taken (\cref{sec:geoloc}), while in the second, the model estimates a distribution over all possible locations (\cref{sec:probabilistic}). Since probabilistic visual geolocation is a novel task, we introduce new metrics and baselines for evaluation.

\noindent We consider three datasets of geolocated images:
\begin{compactitem}
    \item {\bf OpenStreetView-5M \cite{astruc2024openstreetview} (OSV-5M)} contains 5 million street view training images from 225 countries and over 70K cities worldwide. The test set includes 200K images and is built with a $1$km buffer with the train set.
    \item {\bf iNat21 \cite{van2021benchmarking}} includes 2.7 million images of animals from 10K species, collected and annotated by community scientists. We use the public validation set that contains $10$ images for each of the 10K species featured.
     \item {\bf YFCC \cite{YFCC}} The Yahoo Flickr Creative Commons dataset comprises 100 million highly diverse media objects, of which we use the subset of 48 million images with precise geotags. To allow comparison with other methods, we evaluate all methods on the public subset YFCC4k of $4000$ images introduced in \cite{Vp2017revisiting}.
\end{compactitem}

\paragraph{Baselines. } We implement several generative baselines to contextualize our results:
\begin{compactitem}
\item {\bf Uniform.} This baseline assigns a constant density probability of $1 / (4\pi)$ steradian$^{-1}$ to any point on Earth.
\item {\bf von Mises-Fisher Regression \cite{fisher1953dispersion, hasnat2017mises}.} We modify our model to map the image feature to parameters $(\mu, \kappa)$ of a von Mises-Fisher (\textbf{vMF}) distribution on the sphere, where $\mu \in \mathbb{R}^3$, $|\mu| = 1$, and $\kappa > 0$. The network is trained to minimize the negative log-likelihood at the true location $x_0$: 
\begin{align}
    \ell_\text{vMF}(x_0,c) &= 
    -\log_2\left(\vmf(x_0 \mid \mu, \kappa)\right)\\\nonumber
    &= -\log_2\left(\frac{\kappa}{
    4\pi\sinh(\kappa)}\right) - \kappa \mu^\intercal x_0~.
\end{align}
\item {\bf Mixture of vMF \cite{izbicki2020exploiting}.} To handle multimodal distributions, we extend the model to predict a mixture of $K$ vMF distributions (\textbf{vMFMix}) with mixture weights $w_1, \dots, w_K > 0$ and $\sum_{k=1}^K w_k = 1$, and distribution parameters $(\mu_1, \dots, \mu_K, \kappa_1, \dots, \kappa_K)$. The loss is defined as: 
\begin{align}
\!\!\!\!\!\!\!\!\!\!\!  \ell_\text{vMFMix}(x,c) &= -\log_2\left( \sum_k w_i \vmf(x \mid \mu_k, \kappa_k)\right)~. 
\end{align}
\end{compactitem}

\paragraph{Model Parameterization.}
We evaluate our three generative approaches: diffusion and flow matching in $\mathbb{R}^3$ (\textbf{Diff} $\mathbb{R}^3$ and \textbf{FM} $\mathbb{R}^3$), and Riemannian Flow-Matching on the sphere (\textbf{RFM} $\mathcal{S}_2$). All models and baselines are trained on the training set of the dataset they are evaluated on. All models are trained for one million iterations, except \textbf{RFM}$_\textbf{10M}$ $\mathcal{S}_2$ which undergoes 10M iterations.

All models and baselines share the same backbone $\phi$ : a DINOv2-L~\cite{oquab2024dinov2} with registers~\cite{darcet2023vision}, except when training on OpenStreetView-5M, where we employ a ViT-L model~\cite{dosovitskiy2020image} fine-tuned with StreetCLIP (\textbf{SC})~\cite{haas2023learning}. All models use the same configuration for the network $\psi$ with {36} M parameters, except for iNat21, where we use a smaller version with {9.2} M parameters (details in the appendix).
We set the guidance scale to $2$ when predicting locations and to $0$ when computing distributions, as justified in Section~\ref{sec:probabilistic}.

\noindent  

\subsection{Visual Geolocation Performance}
\label{sec:geoloc}

\begin{figure}[t]
    \centering
    \input{figures/timesteps_plot}
    \vspace{-3mm}
    \caption{{\bf Impact of Number of Timesteps.} We represent different metrics on OpenStreetView-5M with different numbers of timesteps for the Riemannian Flow matching model.}
    \label{tab:timestpes}
\end{figure}
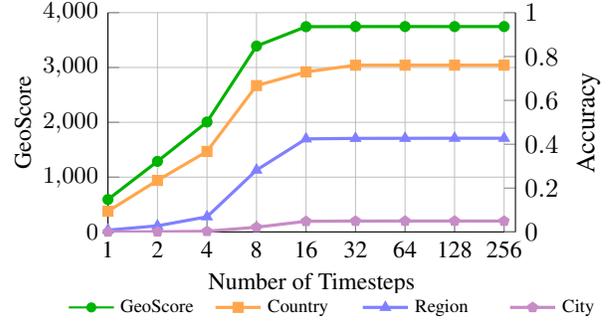

\begin{table*}[t]
    \caption{{\bf Probabilistic Visual Geolocation.} We evaluate the quality of the predicted distributions. Note that the likelihoods of distributions defined in $\bR^3$ and $\cS_2$ are not directly comparable, as they are based on different metrics. Moreover, contrary to the discrete case, log-likelihoods and entropies of continuous distribution can be negative. To save space, we only provide the generation metrics for iNat21.
  }
    \vspace{-1mm}
    \centering
    \input{tables/density}
    \label{tab:density}
\end{table*}

We first evaluate our model's ability to predict the location where an image was taken, comparing its performance to existing geolocation methods from the literature.

\paragraph{Metrics.} We use the following geolocation metrics, averaged across the test sets:
\begin{compactitem}
    \item {\bf Distance:} The Haversine distance (in km) between the true and predicted locations. 
    \item{ \bf GeoScore:}  A score inspired by the game GeoGuessr, defined as $5000\exp(-\delta/1492.7)$ \cite{haas2023pigeon} where $\delta$ is the Haversine distance. This score ranges from $0$ to $5000$, with higher scores indicating better accuracy. 
    \item {\bf Accuracy}: The proportion of predictions that fall within the right countries, regions, or cities, or a set distance to their true location. 
\end{compactitem}

\paragraph{Results.} Table~\ref{tab:geoloc} compares our models against established geolocation methods---including classification, regression, and retrieval-based approaches---as well as our own generative baselines introduced in \cref{sec:probabilistic}. On all three datasets, our models achieve state-of-the-art geolocation performance, beating not only discriminative methods but also retrieval-based approaches that rely on large, million-image databases.

On the large-scale YFCC dataset, extending the training of our best model (RFM $\mathcal{S}_2$) to 9 million iterations yields consistent improvements. Overall, our generative approach surpasses all methods not based on retrieval or prototypes by a considerable margin. Compared to the specialized hybrid approach of Astruc \etal~\cite{astruc2024openstreetview}, we increase the GeoScore by 406 points, reduce the average distance by 745 km, and improve country-level accuracy by 8.2\%. While our methods display excellent results at various scales (from country-level down to 25 km), retrieval-based techniques maintain an advantage at extremely fine-grained resolutions, thanks to their extensive image databases.

Among the generative strategies, flow matching consistently outperforms diffusion, and the Riemannian variant on the sphere outperforms the Euclidean counterpart, highlighting the benefit of incorporating the Earth's geometry into the model. The single-component vMF model performs similarly to a discriminative regression baseline, which aligns with the fact that predicting a single direction on the sphere is essentially location regression. In contrast, the mixture of vMF distributions overfits the training set, leading to weaker performance.

\paragraph{Analysis.} We represent in Figure~\ref{tab:timestpes} the influence of the number of timesteps on the RFM model's performance. The GeoScore improves from 591 (1 step) to 3744 (16 steps), after which it plateaus around 3746. Similarly, country-level accuracy increases from 9.4\% to 76\%, and city-level accuracy from 0.02\% to 4.8\%. This demonstrates that iterative refinement benefits our model up to a certain point, after which additional steps yield diminishing returns.

\subsection{Probabilistic Visual Geolocation}
\label{sec:probabilistic}

Beyond predicting a single location, our model can estimate a distribution over all possible locations, capturing the inherent uncertainty in visual geolocation. 

\paragraph{Metrics. } We evaluate the quality of the predicted distributions $p(y \mid c)$, where $c$ is an image and $y \in \mathcal{S}^2$ represents any location on the Earth's surface, with the following metrics:
\begin{compactitem}
 \item \textbf{Negative Log-Likelihood (NLL)}: 
 We compute the average negative log-likelihood per-dimension  (see \cite[F]{chen2024riemannian}) of the true locations under the predicted distributions:
 \begin{align}
 \text{NLL} = -\frac{1}{3N} \sum_{i=1}^N \log_2 p(x_i \mid c_i)~, \end{align} 
 where $(x_i, c_i)$ are the true location and image pairs in the test set. This metric quantifies how well the predicted distributions align with the true locations. 
  \item \textbf{Localizability}: We quantify the localizability of an image $c$ as the negative entropy of the predicted distribution: \begin{align} 
  \text{Localizability}(c) = \int_{\mathcal{S}^2} p(y \mid c) \log_2 p(y \mid c) dy~. 
  \end{align}
  We estimate this integral with Monte-Carlo sampling \cite{metropolis1949monte} with 10,000 samples.

\item {\bf Generative Metrics.} we report the classic Precision and Recall metrics ~\cite{kynkaanniemi2019improved}, as well as the more recent Density and Coverage ~\cite{naeem2020reliable}. See the appendix for more details.

\end{compactitem}

\begin{figure*}[]
    \centering
      \input{figures/entropy}
      \vspace{-2mm}
        \caption{{\bf Estimating Localizability.} We use the entropy of the predicted distribution as a proxy for the localizability of images. For each dataset, we present examples of high, medium, and low localizability, which correlate well with human perception.}
    \label{fig:entropy}
\end{figure*}
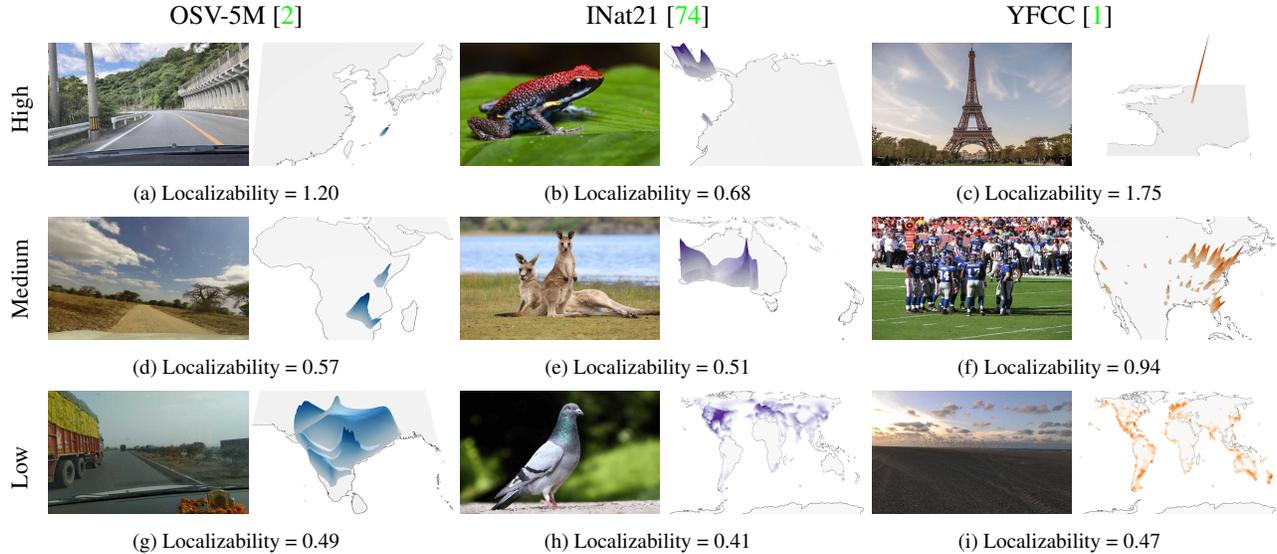

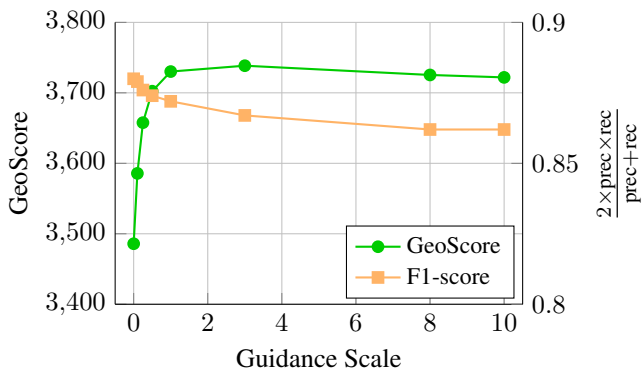
\begin{figure}[b]
    \centering
    \input{figures/cfg_plot}
    \vspace{-2mm}
    \caption{{\bf Impact of Classifier-Free Guidance.} We plot the evolution of the GeoScore and generative metrics depending on the guidance scale $\omega$ for the OSV-5M dataset.}
    \label{fig:cfg_plot}
\end{figure}

\vspace{-4mm}
\paragraph{Results.}

Table~\ref{tab:density} reports the performance of all models for the probabilistic visual geolocation task. Our models achieve significantly lower NLL than the baselines, clearly showing that the predicted distribution is more consistent with the test image locations. Although we cannot directly compare the likelihoods of models defined in $\mathbb{R}^3$ and on the sphere $\mathcal{S}^2$ due to different underlying metrics, we observe that flow matching performed in $\mathbb{R}^3$ yields better NLL than diffusion. The mixture of vMF distributions improves upon the single vMF model across all metrics. This indicates that while mixtures may not enhance geolocation accuracy, they may better capture the inherent ambiguity of the task as many images have multimodal distributions with several reasonable guesses, for example, Ireland \vs New Zeland.

In terms of generative metrics, our Riemannian Flow Matching model outperforms all baselines and models operating in $\mathbb{R}^3$, demonstrating the effectiveness of modeling distributions on the Earth's surface.
We hypothesize that our Riemannian flow matching approach leads to better performances because the results are directly output by the generation process, compared to $\mathbb{R}^3$ where the output of the generation process has to be projected onto $\mathcal{S}^2$ which can add subtle errors.

\paragraph{Localizability.} Figure~\ref{fig:entropy} displays examples of images with low, medium, and high localizability as measured by the negative entropy of the distributions predicted by the Riemannian Flow Matching approach.
The model can detect subtle hints, such as road signage (a) or vegetation (d), to locate street-view images with relatively good confidence.
However, a rural road in India (g) has a low localizability score, as it could have been taken anywhere in the country.
The localizability of animal images (b,e,h) is lower than that of human-centric or street-view images and correlates with the rarity of the species depicted.
Impressively, some images can be pinpointed to a meter-level accuracy, such as the picture of the Eiffel Tower (c). An image captured inside an NFL stadium (f) produces a multimodal distribution centered around major American cities with prominent NFL teams. A picture of a featureless beach (i) results in a highly spread-out distribution along most of the Earth's coastlines, resulting in low localizability.

\paragraph{Guidance.} Figure~\ref{fig:cfg_plot} shows the impact of the guidance scale $\omega$ on both GeoScore and generative metrics---using the F1-score to combine precision and recall. %

At higher guidance scales, the predicted distributions become sharper, concentrating probability density around the predicted locations and assigning less density elsewhere. This increased focus on the modes enhances geolocation accuracy but leads to poorer coverage of the true distribution, as the model collapses onto the most probable areas. Consequently, metrics that evaluate discrepancies between the predicted and true densities—such as precision and recall—worsen at higher guidance levels. This trade-off highlights the balance between achieving high geolocation accuracy and capturing the full diversity of the data distribution.

%% file: figures/timesteps_plot.tex
\begin{tikzpicture}
    \pgfplotsset{
        scale only axis,
        axis y line*=left,
        axis x line*=bottom,
    }
    
    \begin{axis}[
        width=0.65\columnwidth,
        height=0.35\columnwidth,
        xlabel={Number of Timesteps},
        ylabel={GeoScore},
        xmode=log,
        grid=major,
        xmin=1,
        xmax=300,
        ymin=0,
        ymax=4000,
        ytick={0,1000,2000,3000,4000},
        xtick={1,2,4,8,16,32,64,128,256},
        xticklabels={1,2,4,8,16,32,64,128,256},
        ylabel style={yshift=-0.1cm},
        xlabel style={yshift=0.1cm},
        tick label style={font=\small},
        label style={font=\small},
        grid style={line width=.1pt, draw=gray!20},
        major grid style={line width=.2pt,draw=gray!50}
    ]
    
    \addplot[very thick,green!70!black,mark=*,mark size=1.5pt] coordinates {
        (1, 591.4257202148438)
        (2, 1286.9871826171875)
        (4, 2005.494384765625)
        (8, 3390.97314453125)
        (16, 3744.232177734375)
        (32, 3746.4482421875)
        (64, 3746.9267578125)
        (128, 3746.485595703125)
        (256, 3746.93408203125)
    };
    
    \end{axis}
    
    \begin{axis}[
        width=0.65\columnwidth,
        height=0.35\columnwidth,
        scale=1,
        axis y line*=right,
        axis x line=none,
        ylabel={Accuracy},
        ylabel near ticks,
        ymin=0,
        ymax=1,
        ytick={0,0.2,0.4,0.6,0.8,1.0},
        xmin=1,
        xmax=300,
        xmode=log,
        legend style={
            at={(-0.1,-0.28)},
            anchor=north west,
            font=\scriptsize,
            cells={anchor=north},
            draw=none,
            legend columns=4,
            column sep=0.0005\columnwidth,
            inner sep=0pt,
            legend cell align=center,
            nodes={text width=0.14\columnwidth}
        }
    ]
    
    \addlegendimage{thick,green!70!black,mark=*,mark size=1.5pt}
    \addlegendentry{GeoScore}
    
    \addplot[very thick,orange!70!white,mark=square*,mark size=1.5pt] coordinates {
        (1, 0.09426175802946092)
        (2, 0.2348491996526718)
        (4, 0.3671123683452606)
        (8, 0.6674569845199585)
        (16, 0.72994873046875)
        (32, 0.7605355978012085)
        (64, 0.7606154084205627)
        (128, 0.7606483101844788)
        (256, 0.7607046365737915)
    };
    \addlegendentry{Country}
    
    \addplot[very thick,blue!50!white,mark=triangle*,mark size=1.5pt] coordinates {
        (1, 0.00811767578125)
        (2, 0.02777569182217121)
        (4, 0.06889930367469788)
        (8, 0.2819683253765106)
        (16, 0.4243633449077606)
        (32, 0.42659348249435425)
        (64, 0.4269644021987915)
        (128, 0.4273400008678436)
        (256, 0.4271991550922394)
    };
    \addlegendentry{Region}
    
    \addplot[very thick,violet!45!white,mark=pentagon*,mark size=1.5pt] coordinates {
        (1, 0.0002159705472877249)
        (2, 0.0007699819980189204)
        (4, 0.003004807746037841)
        (8, 0.021531324833631516)
        (16, 0.04843843728303909)
        (32, 0.049541767686605453)
        (64, 0.04974834620952606)
        (128, 0.049724873155355453)
        (256, 0.04958402365446091)
    };
    \addlegendentry{City}
    
    \end{axis}
\end{tikzpicture} 

%% file: tables/density.tex
\footnotesize{
\begin{tabular}{l c c ccccc}
\toprule
& \multicolumn{1}{c}{OSV-5M} & \multicolumn{1}{c}{YFCC}  & \multicolumn{5}{c}{iNat21}\\ \cmidrule(lr){2-2}\cmidrule(lr){3-3}\cmidrule(lr){4-8}
& NLL $\downarrow$ & NLL $\downarrow$ & NLL $\downarrow$ & precision $\uparrow$ & recall $\uparrow$ & density $\uparrow$ & coverage $\uparrow$ \\\midrule
Uniform & 1.22 & 1.22 & 1.22 & 0.58 & \bf 0.98 & 0.38 & 0.22\\
\rowcolor{gray!10} vMF Regression & 10.13 & 0.01 & 1.99 & 0.52 & \bf 0.98 & 0.37 & 0.24\\
vMFMix & 0.06 & -0.04 & -0.23 & 0.63 & \bf 0.98 & 0.47 & 0.29\\
\rowcolor{gray!10} \bf RFlowMatch $\cS_2$  (ours)  \bf & \bf  -1.51 & \bf  -3.71 & \bf -1.94 & \bf 
 0.88 & 0.95 & \bf  0.78 & \bf 
 0.59\\
\greyrule
\bf Diffusion $\bR^3$ (ours) & 0.58 & 0.63 & 0.68 & 0.76 & \bf 0.98 & 0.60 & 0.44\\
\rowcolor{gray!10}\bf FlowMatch $\bR^3$ (ours) & \bf -5.01 & \bf -7.15 & \bf -4.00 & 0.76 &  0.97 & 0.61 & 0.47\\
\end{tabular}
}

%% file: figures/entropy.tex
\begin{tabular}{l@{}c@{\,}c@{\,}c}
     &OSV-5M \cite{astruc2024openstreetview}
     &
     INat21 \cite{van2021benchmarking}
     &
     YFCC \cite{YFCC} \\
\rotatebox{90}{\qquad\;\; \small High }
&\begin{subfigure}{0.31\textwidth}
\begin{tabular}{c@{\,}c}
\includegraphics[width=.49\linewidth, height=0.3\linewidth]{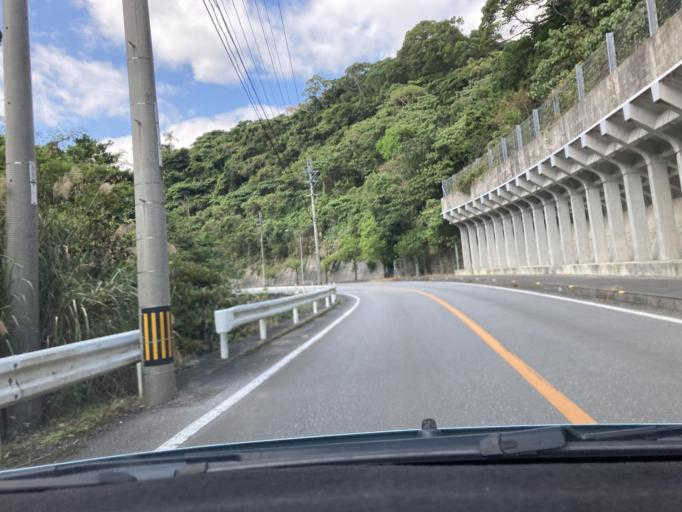}
&
\includegraphics[width=.49\linewidth, height=0.3\linewidth]{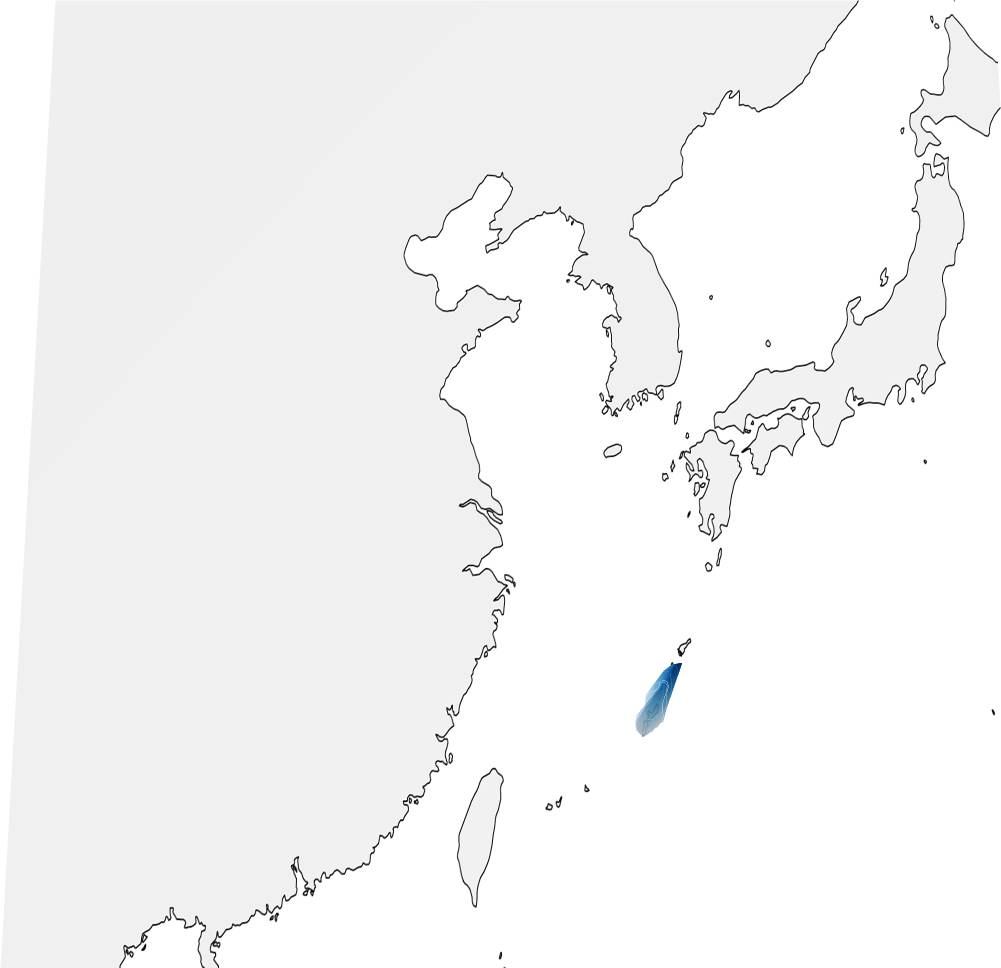}
\end{tabular}
\caption{Localizability = 1.20}
\end{subfigure}
&\begin{subfigure}{0.31\textwidth}
\begin{tabular}{c@{\,}c} 
\includegraphics[width=.49\linewidth, height=0.3\linewidth]{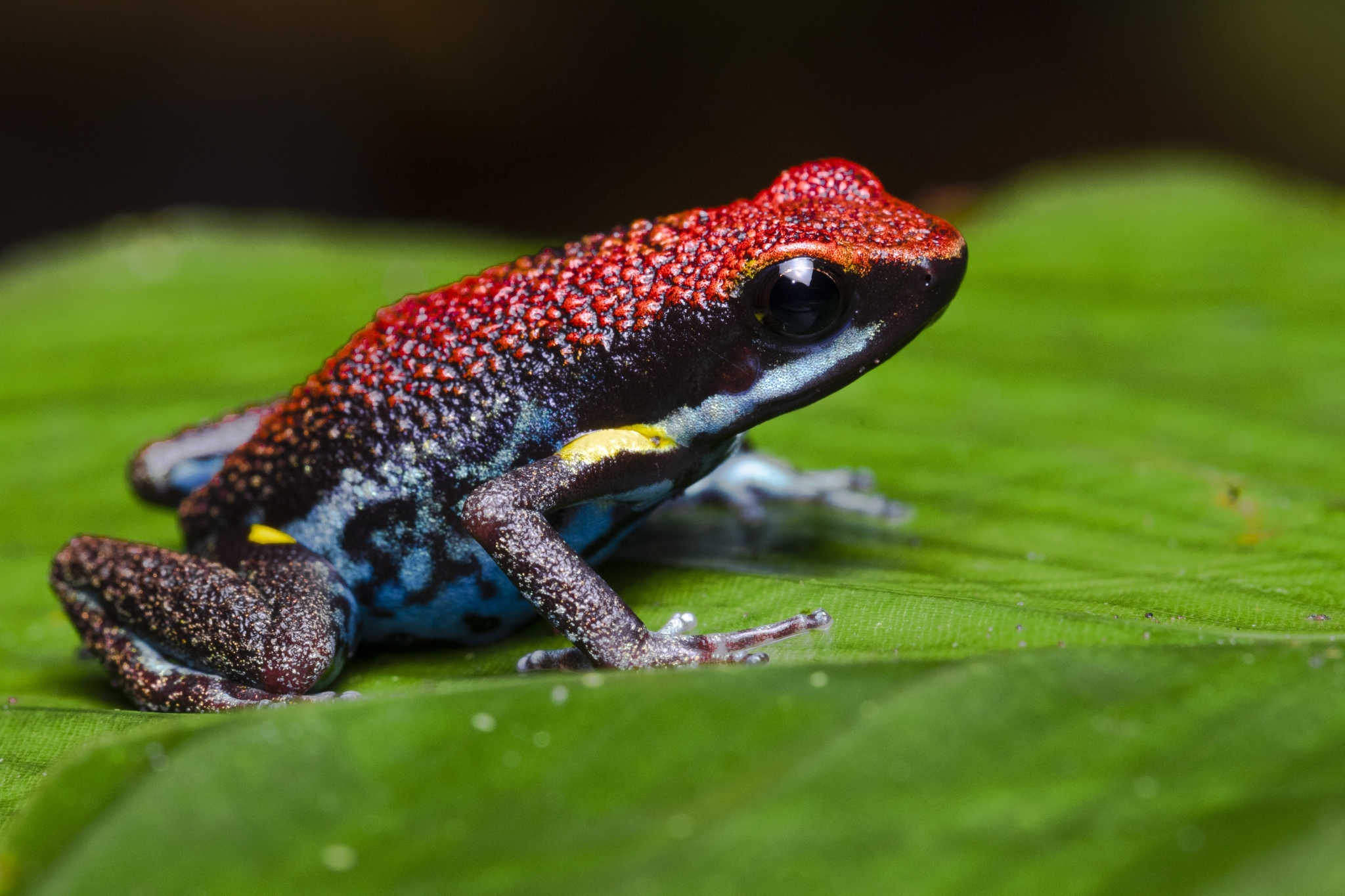} & \includegraphics[width=.49\linewidth, height=0.3\linewidth]{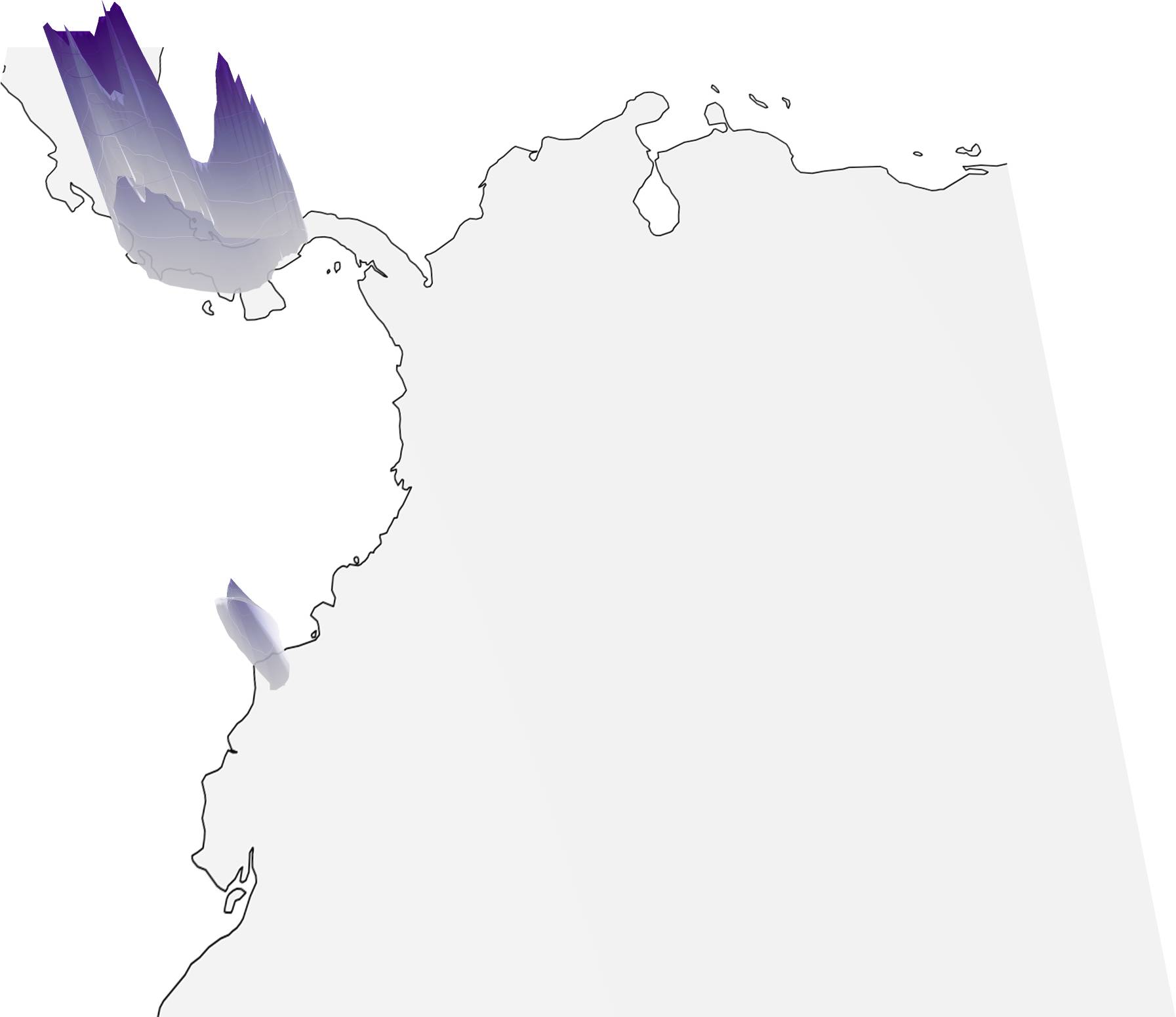}
\end{tabular}
\caption{Localizability = 0.68}
\end{subfigure}
&\begin{subfigure}{0.31\textwidth}
\begin{tabular}{c@{\,}c} 
\includegraphics[width=.49\linewidth, height=0.3\linewidth]{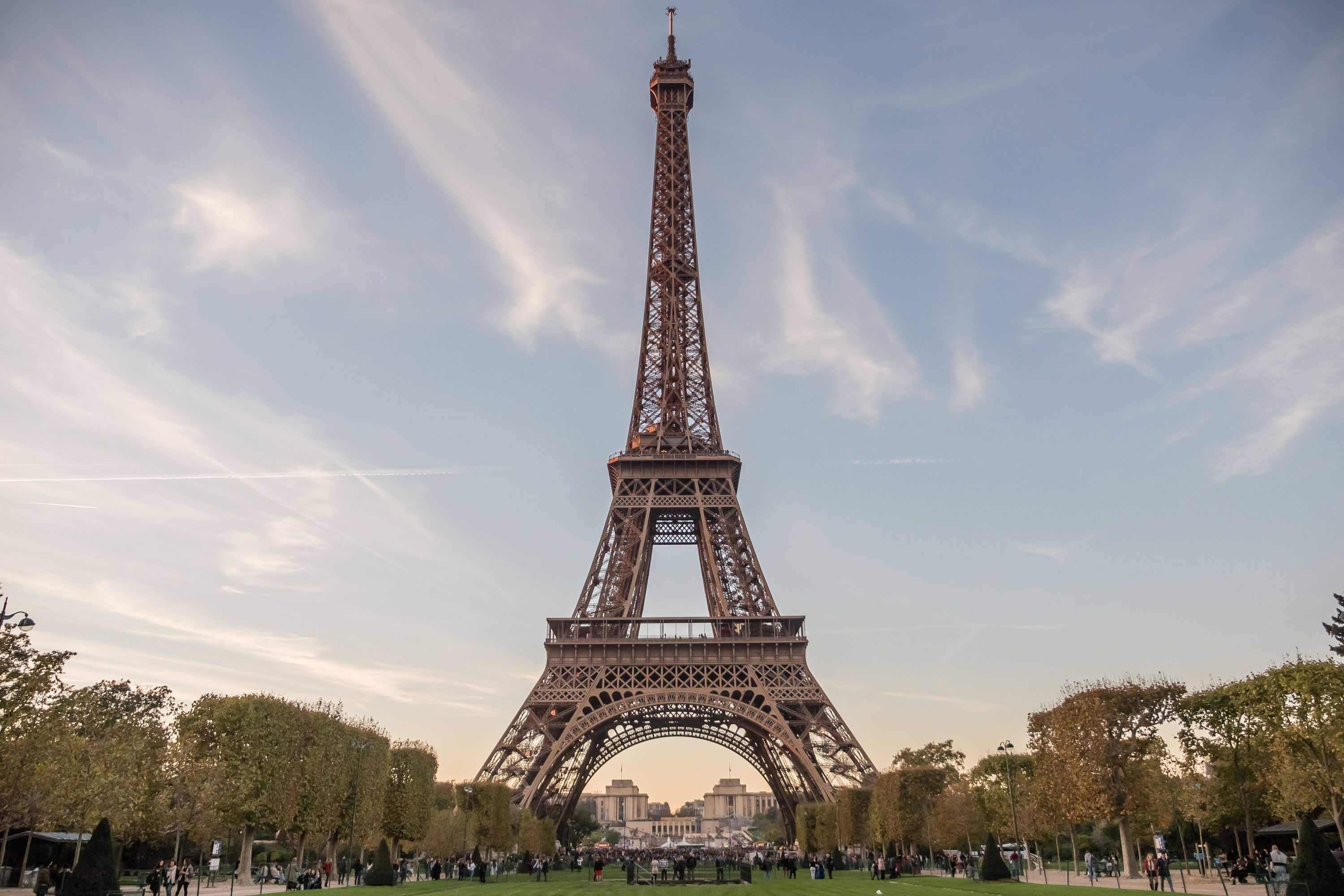} & 
\centering
\begin{tikzpicture}
\node[draw=none] (x) at (0,0) {};
\node[inner sep=0pt, anchor=south west] (img) at (0.3,0){\includegraphics[width=.35\linewidth, height=0.3\linewidth]{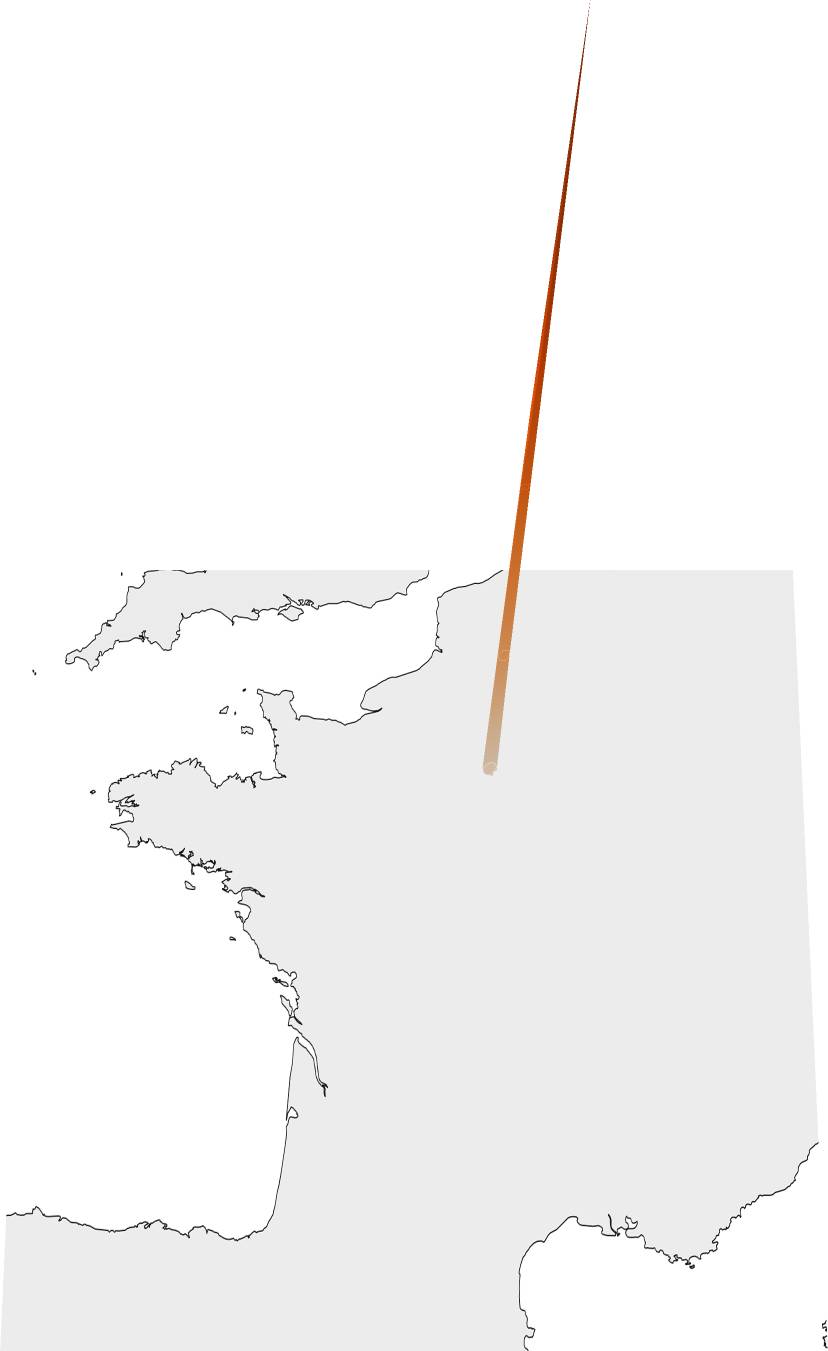}};
\end{tikzpicture}
\end{tabular}
\caption{Localizability = 1.75}
\end{subfigure}\\
\rotatebox{90}{\quad\;\;\;\; \small Medium }
&\begin{subfigure}{0.31\textwidth}
\begin{tabular}{c@{\,}c}
\includegraphics[width=.49\linewidth, height=0.3\linewidth]{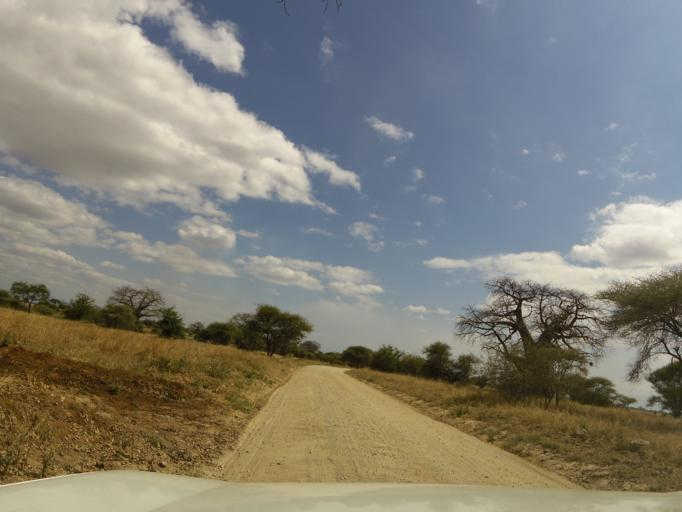} & \includegraphics[width=.49\linewidth, height=0.3\linewidth]{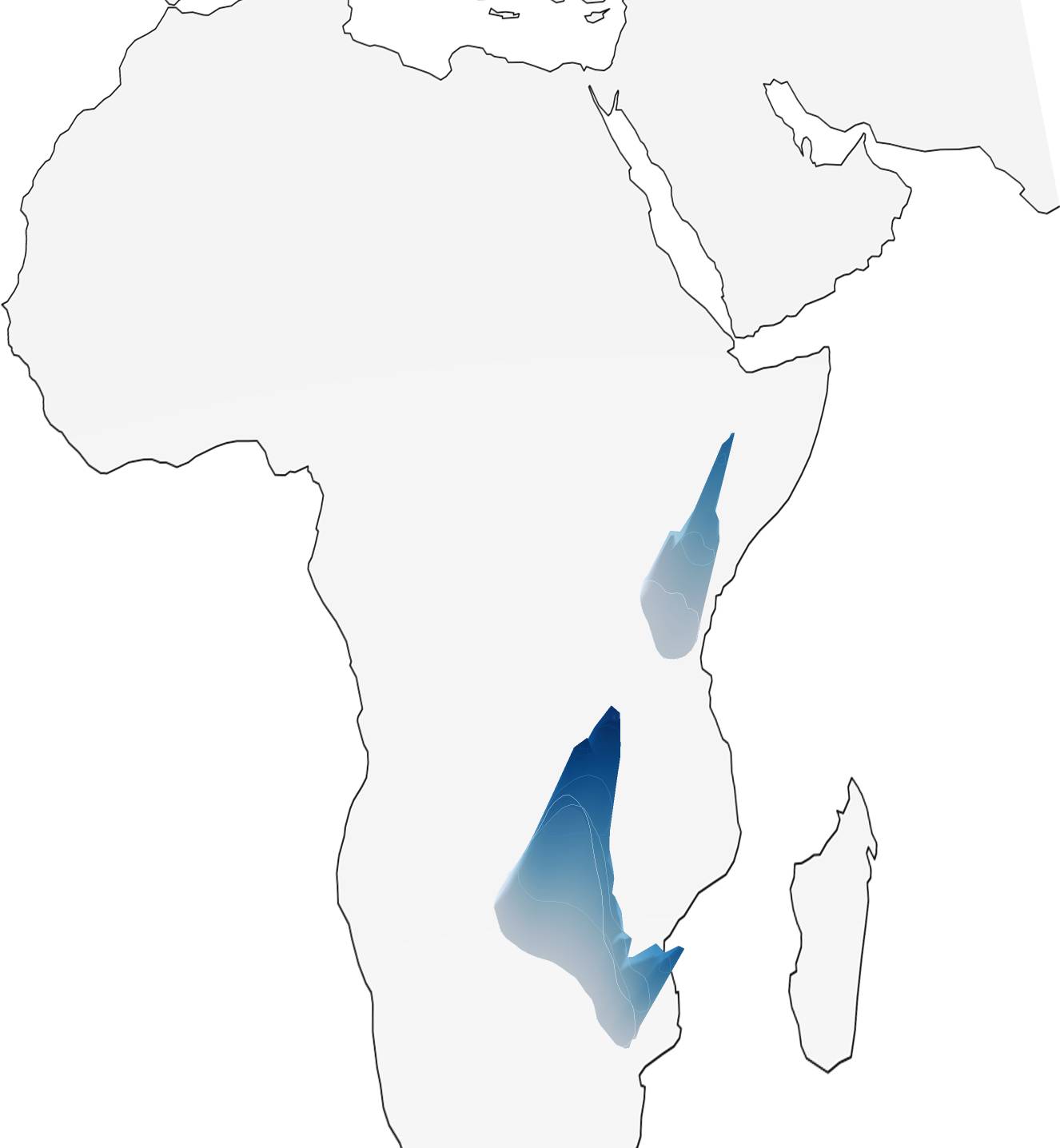} 
\end{tabular}
\caption{Localizability = 0.57}
\end{subfigure}
&\begin{subfigure}{0.31\textwidth}
\begin{tabular}{c@{\,}c} 
 \includegraphics[width=.49\linewidth, height=0.3\linewidth]{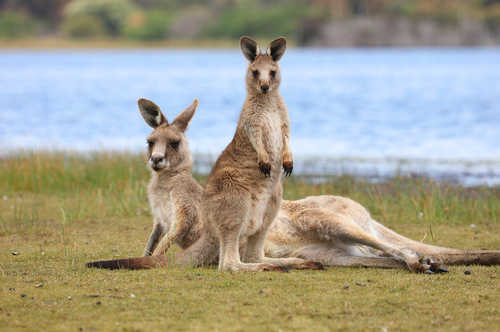} & \includegraphics[width=.49\linewidth, height=0.3\linewidth]{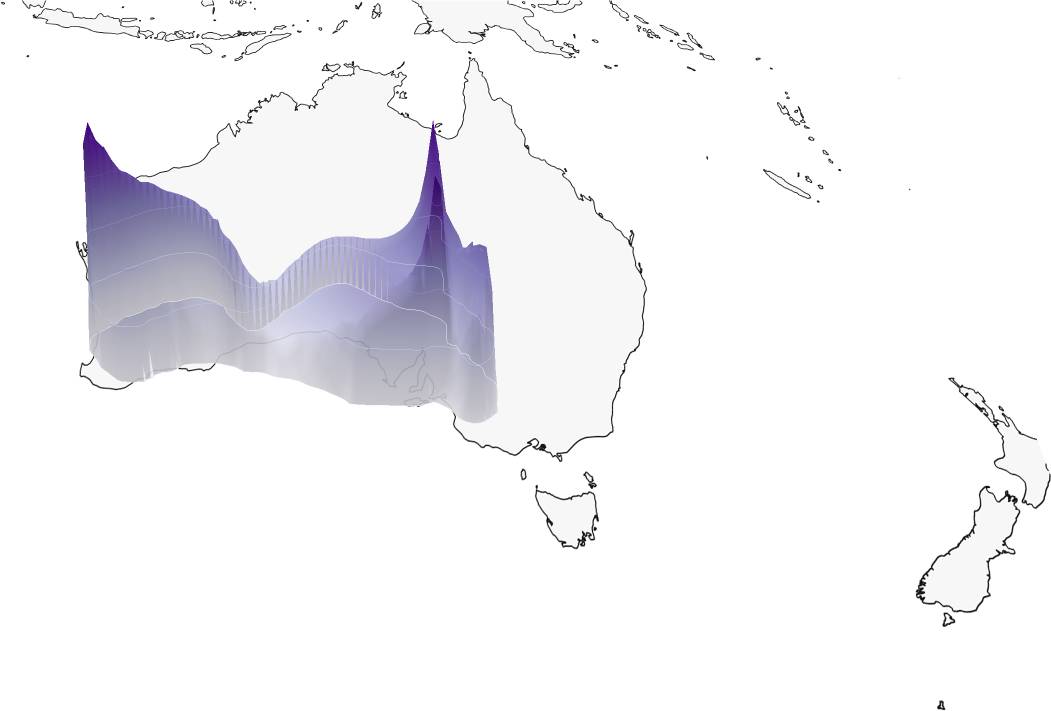}
\end{tabular}
\caption{Localizability = 0.51}
\end{subfigure}
&\begin{subfigure}{0.31\textwidth}
\begin{tabular}{c@{\,}c}
\includegraphics[width=.49\linewidth, height=0.3\linewidth]{images/yfcc_med.jpg} & \includegraphics[width=.49\linewidth, height=0.3\linewidth]{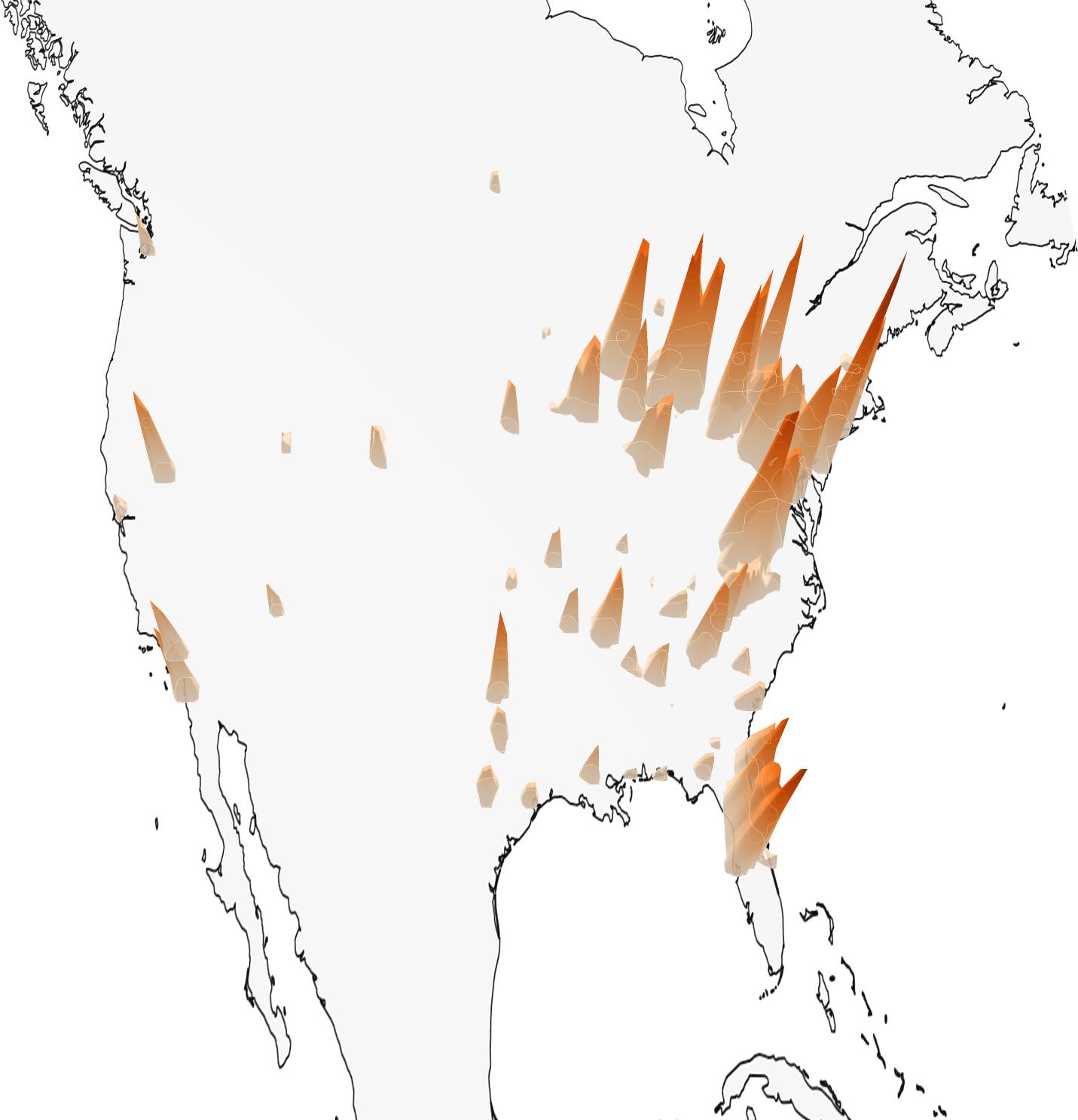} 
\end{tabular}
\caption{Localizability = 0.94}
\end{subfigure}
     \\
\rotatebox{90}{\qquad\; \small Low }
&\begin{subfigure}{0.31\textwidth}
\begin{tabular}{c@{\,}c} \includegraphics[width=.49\linewidth, height=0.3\linewidth]{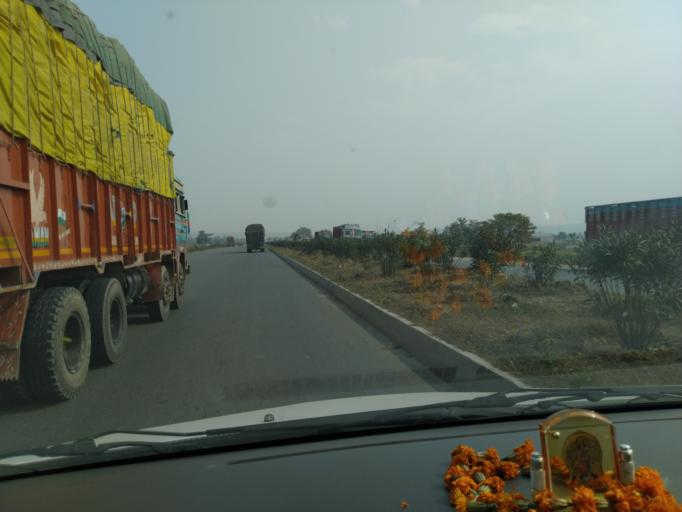} & \includegraphics[width=.49\linewidth, height=0.3\linewidth]{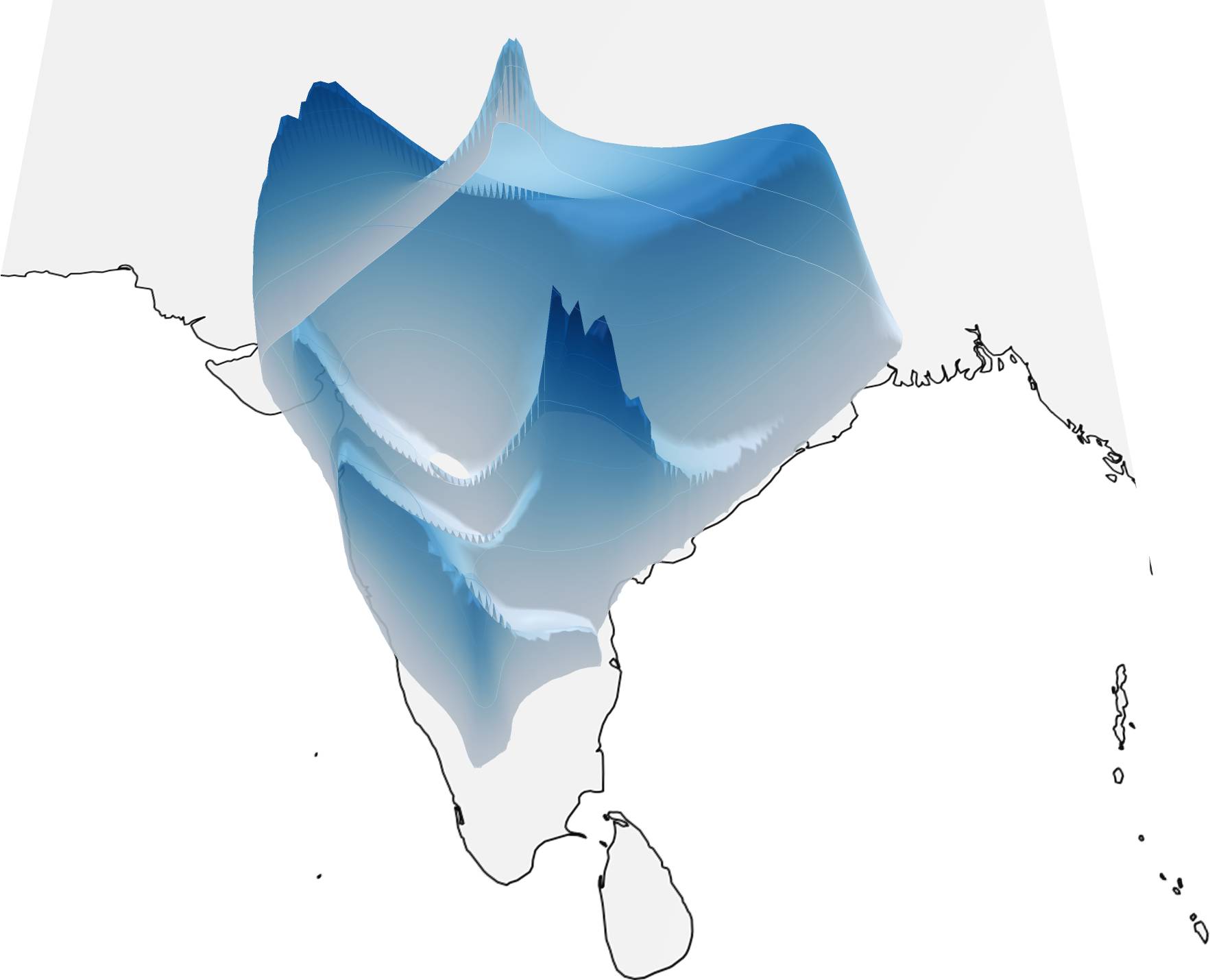} \end{tabular}
\caption{Localizability = 0.49}
\end{subfigure}
&\begin{subfigure}{0.31\textwidth}
\begin{tabular}{c@{\,}c} 
\includegraphics[width=.49\linewidth, height=0.3\linewidth]{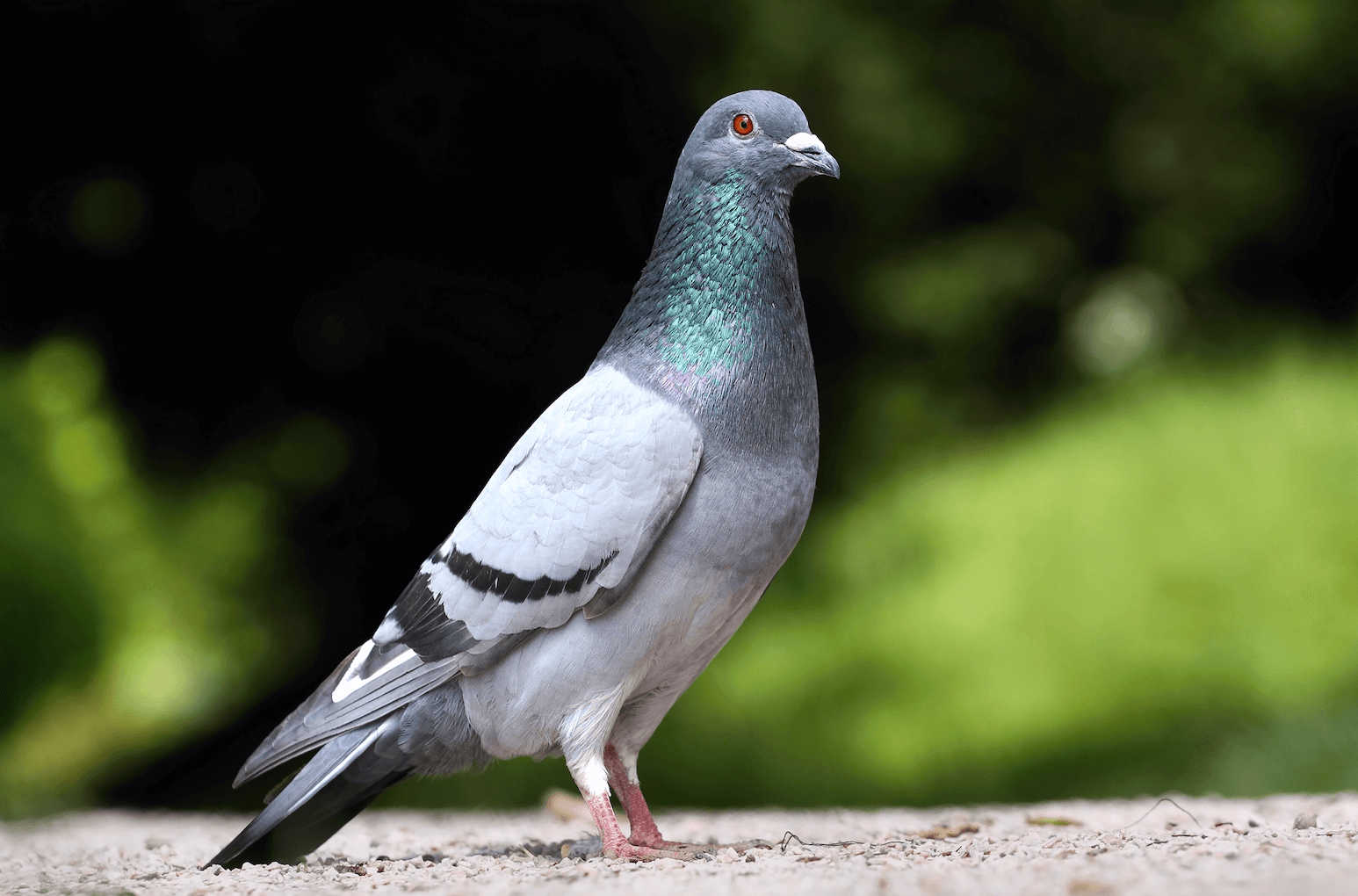} & \includegraphics[width=.49\linewidth, height=0.3\linewidth]{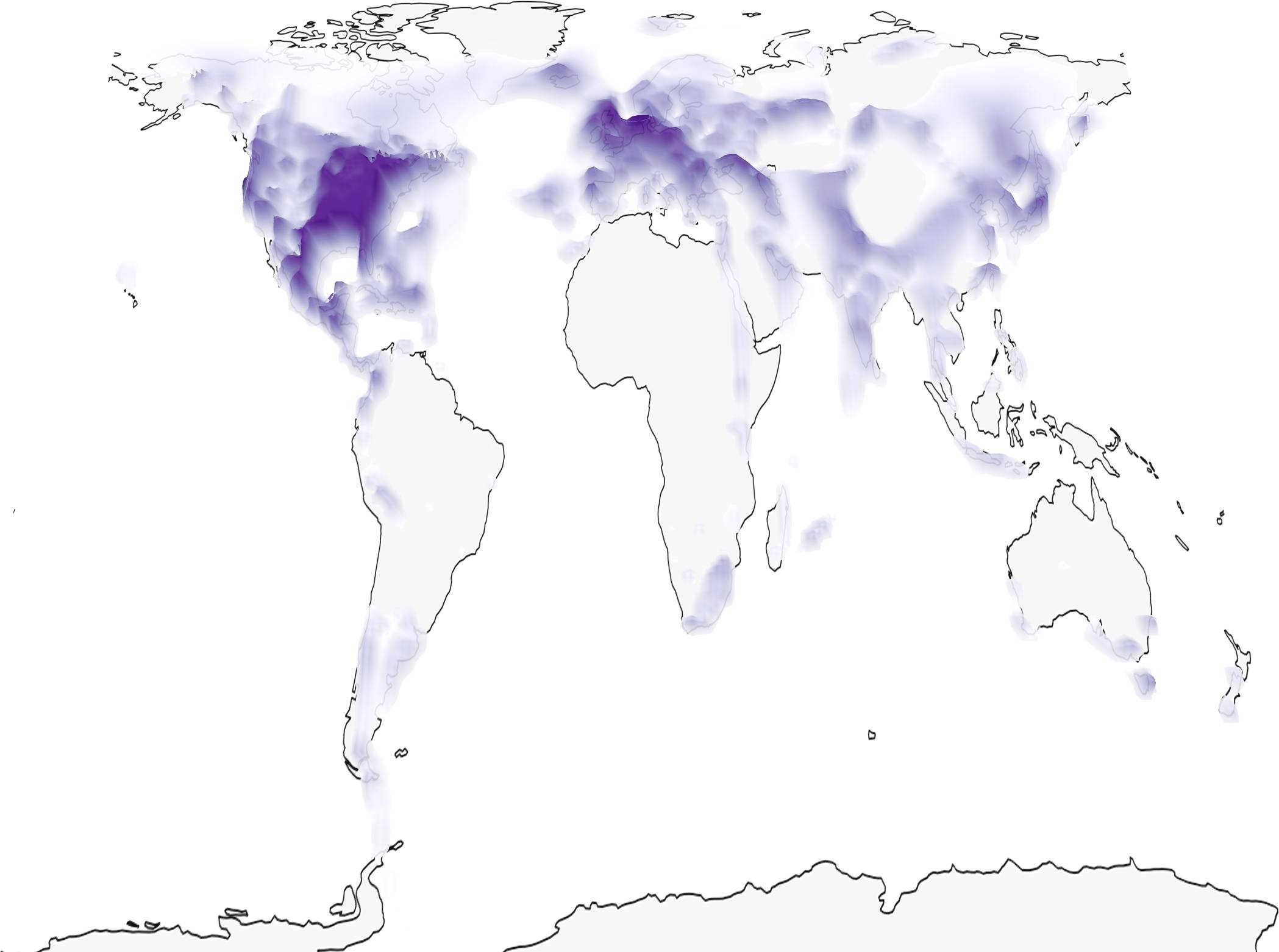}
\end{tabular}
\caption{Localizability = 0.41}
\end{subfigure}
&\begin{subfigure}{0.31\textwidth}
\begin{tabular}{c@{\,}c} \includegraphics[width=.49\linewidth, height=0.3\linewidth]{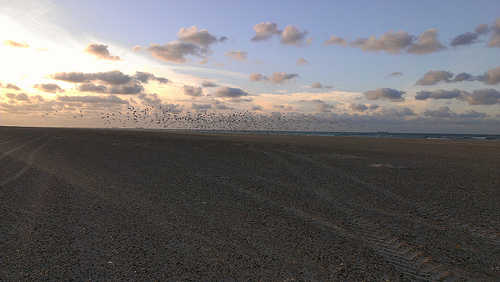} & \includegraphics[width=.49\linewidth, height=0.3\linewidth]{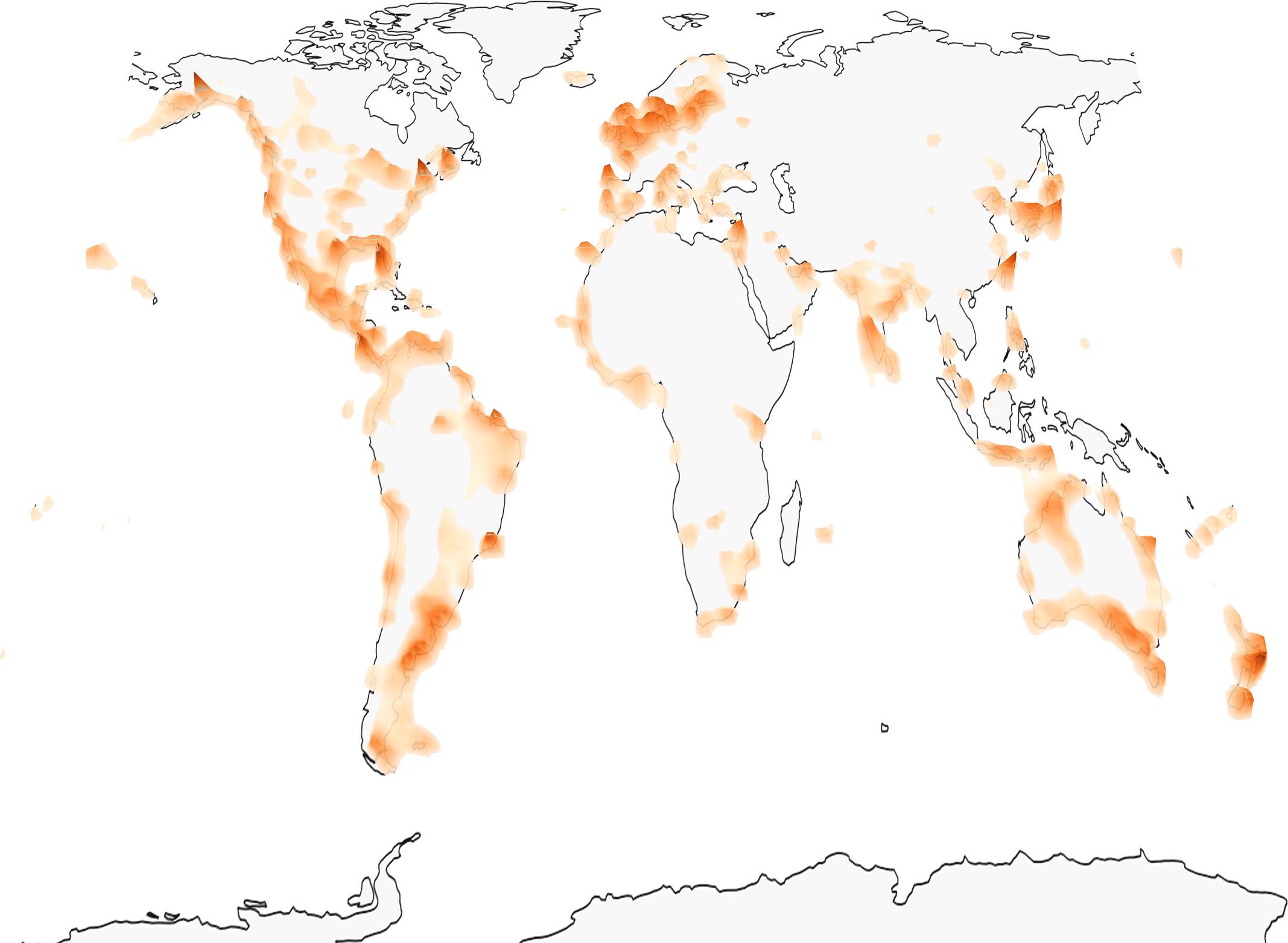} \end{tabular}
\caption{Localizability = 0.47}
\end{subfigure}
\end{tabular}

%% file: figures/cfg_plot.tex
\begin{tikzpicture}
    \pgfplotsset{
        scale only axis,
        axis y line*=left,
        axis x line*=bottom,
    }
    
    \begin{axis}[
        width=0.65\columnwidth,
        height=0.45\columnwidth,
        xlabel={Guidance Scale},
        ylabel={GeoScore},
        xmin=-0.5,
        xmax=10.5,
        ymin=3400,
        ymax=3800,
        grid=major,
        grid style={line width=.1pt, draw=gray!20},
        major grid style={line width=.2pt,draw=gray!50},
        legend style={
            anchor=south east,
            at={(0.98,0.02)},
            font=\small,
            legend columns=1,
            cells={anchor=west},
        },
        xlabel near ticks,
        ylabel near ticks,
        xtick={0,2,4,6,8,10},
        ytick={3400,3500,3600,3700,3800},
    ]
        
    \addplot[mark=*, color=green!80!black, thick] coordinates {
        (0.0, 3485.88)
        (0.1, 3585.70)
        (0.25, 3657.69)
        (0.5, 3702.34)
        (1.0, 3730.19)
        (3.0, 3738.49)
        (8.0, 3725.36)
        (10.0, 3721.92)
    };
    \addlegendentry{GeoScore}
    
    \addplot[mark=square*, color=orange!60!white, thick] coordinates {
        (0.0, -1.81)
        (0.1, -0.56)
        (0.25, 3.30)
        (0.5, 8.83)
        (1.0, 16.09)
        (3.0, 33.09)
        (8.0, 56.58)
        (10.0, 62.95)
    };
    \addlegendentry{F1-score}
    
    \end{axis}
    
    \begin{axis}[
        width=0.65\columnwidth,
        height=0.45\columnwidth,
        axis y line*=right,
        axis x line=none,
        ylabel={\large$\frac{2\times\text{prec}\times\text{rec}}{\text{prec}+\text{rec}}$},
        ylabel near ticks,
        ymin=0.8,
        ymax=0.9,
        xmin=-0.5,
        xmax=10.5,
        ytick={0.8,0.85,0.9},
    ]
    
    \addplot[mark=square*, color=orange!60!white, thick, forget plot] coordinates {
        (0.0, 0.88)
        (0.1, 0.879)
        (0.25, 0.876)
        (0.5, 0.874)
        (1.0, 0.872)
        (3.0, 0.867)
        (8.0, 0.862)
        (10.0, 0.862)
    };
    \end{axis}
\end{tikzpicture} 

%% file: text/A_suppmat.tex
\renewcommand\thefigure{\Alph{figure}}
\renewcommand\thesection{\Alph{section}}
\renewcommand\thetable{\Alph{table}}
\renewcommand\theequation{\Alph{equation}}
\setcounter{equation}{0}
\setcounter{section}{0}
\setcounter{figure}{0}
\setcounter{table}{0}

\maketitlesupplementary

In this appendix, we present our ablation study in \cref{sec:ablation}, and provide additional results and qualitative illustrations in \cref{sec:quali}. We then provide implementation and technical details in \cref{sec:implem}, and some technical elements and proofs in \cref{sec:proof}.

\begin{figure*}[h]
    \input{figures/quali}
    \caption{{\bf Qualitative Illustration.} We represent the predicted distributions predicted by different models for the same image, taken in an NFL stadium in Maryland, USA.}
    \label{fig:quali}
\end{figure*}
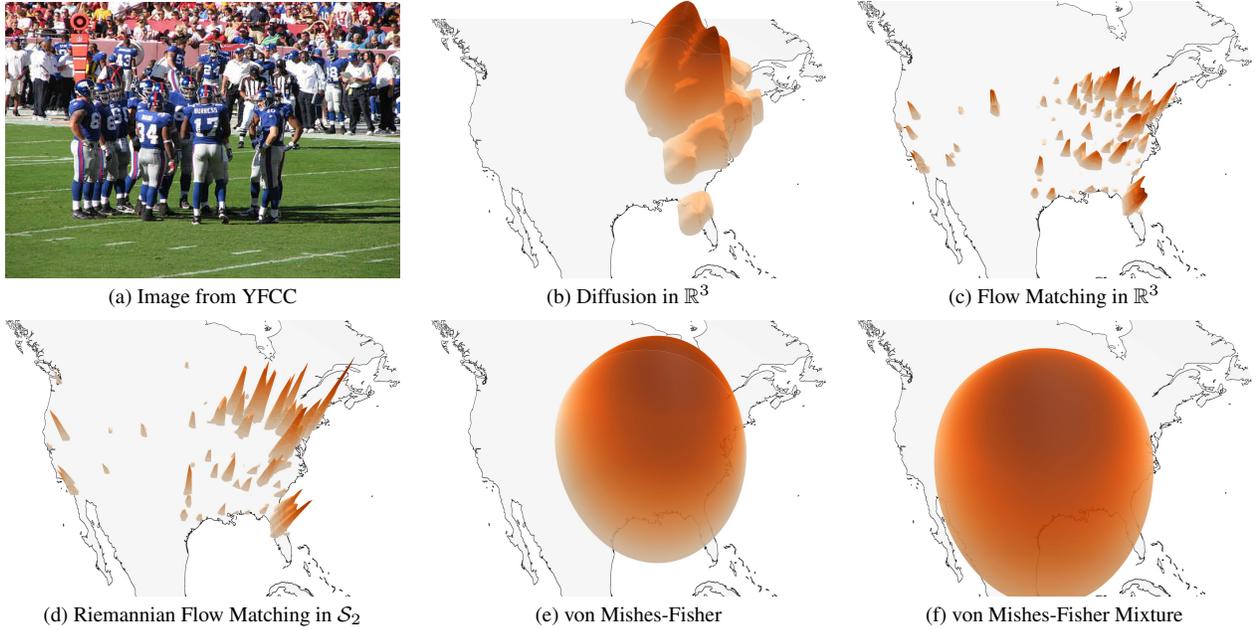

\section{Ablation Study}
\label{sec:ablation}
We conduct an ablation study on the Riemannian Flow Matching approach to evaluate the impact of our design choices, and report the results in \cref{tab:ablation}.

\begin{compactitem}
\item {\bf Guided Sampling.} Guided sampling improves the geoscore, but as shown in Figure 7 of the main paper, leads to low likelihood scores due to overconfident predictions.
    \item \textbf{Single sampling without guidance.} We do not add any guidance ($\omega=0$ in Eq.~13). We observe a loss of geoscore of 182 GeoScore point (3485 vs 3767) , but the NLL is better (-1.8 vs 33.1). Guidance improves the geolocation performance but significantly worsen the probabilistic prediction.
    \item \textbf{Ensemble sampling.} We sample and denoise 32 random points and select the prediction with highest likelihood. While this approach yields the best performance for the distribution estimation metrics, it is significantly more computationally expensive due to the necessity of generating and evaluating multiple samples. In practice, this inflates the prediction time per image from approximately {2 milliseconds} to {72 milliseconds.}.
    \item {\bf Standard Sigmoid Scheduler.} We replace our proposed scheduler defined in Eq.15 of the main paper by the standard not skewed sigmoid scheduler with $\alpha=-3$ and $\beta=3$. This modification increases the geoscore but decreases the quality of the predicted densities as measurs by the generative metrics. The standard sigmoid does not allocate sufficient emphasis to the earlier stages of the diffusion process ($t$ close to 0: low noose regime), which are crucial for fine-grained localization.
    \item {\bf Linear Sigmoid Scheduler.} We replace our proposed scheduler defined in Eq.15 of the main paper by a linear scheduler. This modification decreases both the geoscore and the quality of the predicted densities.
\end{compactitem}

\begin{table*}[t]
    \centering
        \caption{{\bf Ablation Study.} We estimate the impact of different designs. We consider a Riemannian diffusion model and evaluate on  OpenStreetView-5M.}
    \input{tables/ablation}

    \label{tab:ablation}
\end{table*}

\begin{table*}[t]
    \caption{{\bf Generative Metrics.} We evaluate the quality of the predicted distributions with generated metrics for  OSV-5M and YFCC for the unconditional distribution.
  }
    \centering
    \input{tables/detailed_density}
    \label{tab:detailed_density}
\end{table*}

\section{Qualitative Illustration}
\label{sec:quali}

\paragraph{Qualitative Illustrations.}
We provide a detailed illustration of our network in \cref{fig:quali}. We observe that the parametric methods vMF and vMF mixture fail to capture highly multimodal distributions. In contrast, our distributions are non-parametric and can predict highly complex spatial distributions. The vMF mixture is collapse to a single vMF, as we observed for a majority of the prediction.

We observe that both flow matching approaches give results that visually close. Note however that the value of the likelihoods are not comparable as both models are not embedded in the same metric space. The generative metrics detailed in \cref{tab:detailed_density} show that the Riemannian model fits the unconditioned distribution better at a fine-grained scale.

\paragraph{Detailed Quantitative Results.} We provide in \cref{tab:detailed_density} the full generative metrics for the OSV-5M and YFCC datasets. Similarly to what we observed for iNat21 in the main paper, flow matching and particularly Riemannian flow matching leads to the most faithful predicted distributions of samples.

\ifx

\paragraph{Visualizing Velocity Fields.} We represent in \cref{fig:velocity} the velocity fields predicted for different images and at different times of the denoising process. We observe that this field is dynamic as it changes with $t$.
\begin{figure*}
    \input{figures/velocity}

\caption{{\bf Velocity Fields.} We illustrate the velocity fields by the Rimennian Flow Matching model. }
\end{figure*}
\fi

\section{Implementation Details}
\label{sec:implem}

\paragraph{Baseline Details.}
We use the same backbone and image encoder as in our model for all baselines. We adapt them to the baselines with two modifications:
(i) The missing inputs (noisy coordinates and scheduler) are replaced by learnable parameters.
(ii) We replace the final prediction head with MLPs that predict the parameters of the von Mises-Fisher (vMF) distribution: the mean direction $\mu \in \mathcal{S}^2$ (using $L_2$ normalization) and the concentration parameter $\kappa > 0$ (using a softplus activation).

For the mixture of vMF model, we use $K = 3$ vMF distributions. The $\mu$ and $\kappa$ heads now predict three sets of parameters, and the mixture weights are predicted by another dedicated head (with a softmax activation).

\paragraph{Architecture Details.} 
Our model architecture, illustrated in \cref{fig:architecture}, consists of several key components:

\begin{compactitem}
    \item \textbf{Input Processing:} The model takes three inputs: the current coordinate $x_t$, an image embedding $c$, and the noise level $\kappa(t)$.
    
    \item \textbf{Initial Transformation:} The coordinate $x_t$ first passes through a linear layer that expands the dimension from 3 to $d$, followed by an ADA-LN layer that conditions on parameters $\alpha, \beta$.
    
    \item \textbf{Main Processing Block:} The core of the network (shown in gray) is repeated $N$ times and consists of:
        \begin{compactitem}
            \item A linear layer that expands dimension from $d$ to $4d$
            \item A GELU activation function
            \item A linear layer that reduces dimension from $4d$ to $d$
            \item An ADA-LN layer conditioned on $\alpha, \beta$
        \end{compactitem}
    
    \item \textbf{AdaLN:} The AdaLN layer is a conditional layer normalization that scales and shifts the input based on the image features:
    \begin{equation}
        \text{AdaLN}(x) = \gamma \odot \frac{x - \mu}{\sigma} + \beta
    \end{equation}
    where $\mu, \sigma$ are the mean and standard deviation of $x$ on the feature dimension, and $\gamma, \beta$ are learnable parameters.
    
    \item \textbf{Skip Connections:} Each processing block has a skip connection path that:
        \begin{compactitem}
            \item Skips the processing block and directly connects the input to the output to allow a better gradient flow.
            \item Is modulated by a gating parameter $\gamma$ that controls how much of the block output is added to the main path.
        \end{compactitem}
        This gated skip connection allows the network to adaptively control information flow around each processing block.
    
    \item \textbf{Output Head:} The final prediction is obtained through a linear layer that maps to the target dimension $d \mapsto 3$.
    
    \item \textbf{Time step Conditioning:} The noise level $\kappa(t)$ is incorporated through addition to the conditioning of the AdaLN layers.
\end{compactitem}

We use $N=12$ blocks of dimension $d=512$ for OSV-5M and YFCC-100M and blocks of dimension $d=256$ for iNat21.

\paragraph{Optimization.} We train our models for 1M steps with a batch size of 1024, using the Lamb optimizer~\cite{you2020large} with a learning rate of $8*10^{-4}$. We use a warmup of 500 steps and a cosine decay learning rate schedule. We use an EMA of 0.999 for the model weights. For OSV-5M and YFCC-100M, we use a weight decay of $0.05$ and for iNaturalist we use $0.1$. We drop out 10\% of the time the conditioning image embedding to allow classifier free guidance.
\begin{figure}[t]
    \centering
    \input{figures/architecture}
    \caption{{\bf Architecture.} Our model takes as input the current coordinate $x_t$, the image embedding $\phi(c)$, and the noise level $\kappa(t)$. We use this architecture for all our formulations, including deterministic baselines.}
    \label{fig:architecture}
\end{figure}
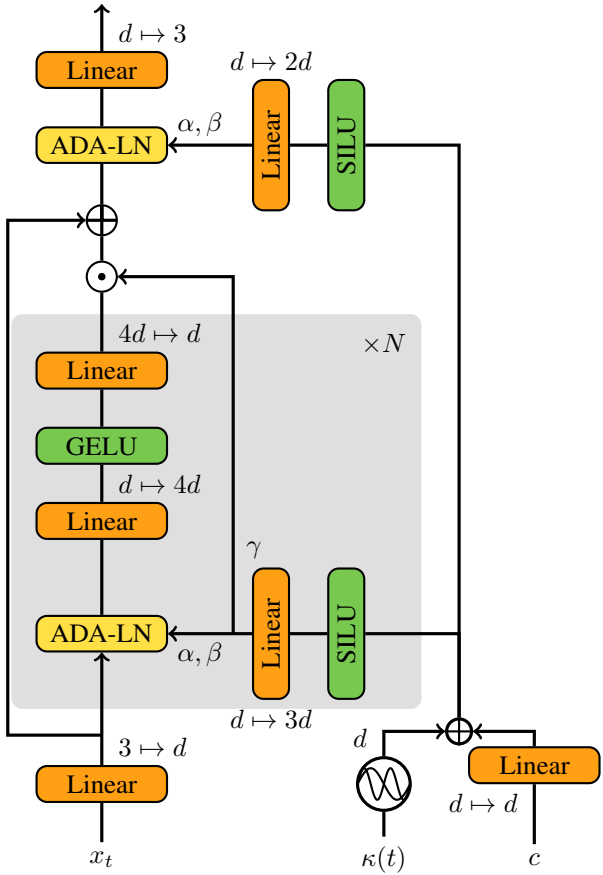

\paragraph{Metrics.}
\begin{compactitem}
    \item \textbf{Precision and Recall:} We adapt the classic generation metrics of precision and recall~\cite{kynkaanniemi2019improved} to our spatial setting by considering geographic proximity.
    
    We consider a set $X$ of true locations, and a set $Y$ of locations sampled from the unconditional distribution predicted by our model. 
    For $Z$ a set of locations ($X$ or $Y$) and $z\in Z$, we define $\BB(z,Z)$ the ball of $\cS_2$ centered on $z$ and with radius equal to the $k$-th nearest neighbour of $z$ in $Z$. We can then define the approximated manifold of ta et of locations:
    \begin{align}
        \mani(Z) := \bigcup_{z \in Z} \BB(z,Z)~.
    \end{align}
    We now define the precision and recall as the proportion of predicted (resp. true) locations within the manifold of true (resp. predicted) locations:    
    \begin{align}
        \text{precision} &:= \frac{1}{\mid Y \mid} \sum_{y \in Y}[y \in \mani(X)]
        \\
        \text{recall} &:= \frac{1}{\mid X \mid} \sum_{x \in X}[x \in \mani(Y)]~,
    \end{align}
    where $[P]$ is the Iverson bracket, equal to one is the statement $P$ is true and $0$ otherwise. Throughout this paper, we select the number of neighbours to $k=3$.
    
    \item \textbf{Density and Coverage:} Naeem \etal \cite{naeem2020reliable} introduce more reliable versions of the  precision and recall metrics, particularly for distributions containing outliers.
    We propose to adapt these metrics to our setting.
    The density measures how closely the predicted locations $Y$  cluster around the true location $X$ :
        \begin{align}
            \text{density} := \frac{1}{k \mid Y \mid} \sum_{y \in Y} \sum_{x \in X} 
            [y \in \BB(x,X)]~.
        \end{align}
    The recall metrics can be misleading high for predicted manifolds that cover uniformly the embeddings space, which is particularly problematic on a low- dimensional space such as $\cS_2$: the uniform distribution has a recall of $0.98$ on OSV-5M.
    Coverage better captures how well the generated distribution spans the true data modes without rewarding such overestimation by assessing how well the predicted distributions span the true data:
        \begin{equation}
            \text{coverage} := \frac{1}{\mid X \mid} \sum_{x \in X} 
            [\exists y \in Y \cap  \BB(x,X)]~.
        \end{equation}   
\end{compactitem}

\section{Technical Details}
\label{sec:proof}

In this section, we present details on Riemannian geometry on the sphere, and a proof sketch of Proposition~1 and elements on its generalization.

\paragraph{Spherical Geometry.} 
The logarithmic map $\log_x$ maps a point $y \in \mathcal{S}_2$ onto $T_x$, the tangent space at point $x$~\cite{sommer2020introduction}:
\begin{align}
    \log_x(y) = \frac{\theta}{\sin \theta}(y-\cos{\theta} x)~,
\end{align}
where $\theta=\arccos(\langle x,y \rangle)$ is the angle between $x$ and $y$.
The exponential map $\exp_x$ of a point $x \in \mathcal{S}_2$ maps a tangent vector $v\in T_x$ back onto th sphere:
\begin{align}
    \exp_x(v) = \cos(\Vert v \Vert) x + \frac{\sin( \Vert v \Vert)}{\Vert v \Vert}v~,
\end{align}
where $\Vert v \Vert$ is the Euclidean norm of $v$.

\paragraph{Proof of Prop~1.} Please find here the corrected proposition and its proof. We now propose a short proof of Proposition~1, inspired by \cite[Appendix C]{lipman2022flow}
\begin{prop}
  Given a location $y \in \mathcal{S}^2$ and an image $c$, consider solving the following ordinary differential equation system for $t$ from $0$ to $1$:
    \begin{align}
    \frac{d}{dt}
     \begin{bmatrix}
     x(t) \\
     f(t)
     \end{bmatrix}
     =
     \begin{bmatrix}
     \psi(x(t) \mid c) \\
     -\divergence\,\psi(x(t) \mid c)
     \end{bmatrix}
     \;\text{with}\;
    \begin{bmatrix}
     x(0) \\
     f(0)
     \end{bmatrix}
     =
     \begin{bmatrix}
     y \\
     0
     \end{bmatrix}~,
    \end{align}
    Then the log-probability density of $y$ given $c$ is:
$\log p(y \mid c) = \log p_\epsilon(x(1) \mid c) + f(1)$ where $p_\epsilon$ is the distribution of the noise $\epsilon$, and $f(t)$ accumulates the divergence of the velocity field along the trajectory.
\end{prop}
\begin{proof}
   The logarithmic mass conservation theorem~\cite{ben2022matching, villani2009optimal} writes:
   \begin{align}
       \frac{d}{dt} \log p(x_t \mid c)
       +
       \divergence\, v(x_t) = 0~.
   \end{align}

After training the network $\psi$ to regress $v(x_t)$, we can substitute $\psi(x_t \mid c)$ to $v(x_t)$ and obtain:
   \begin{align}
       \frac{d}{dt} \log p(x(t) \mid c)
       +
       \divergence\, \psi(x(t) \mid c) = 0~.
   \end{align}
   We integrate from $0$ to $1$:
\begin{align}\label{eq:integral}
       \log p(x_1 \mid c) - \log p(x(0) \mid c) =   
       -\int_0^1 \divergence\, \psi(x(t) \mid c)~.
\end{align}
We thus have the following system:
 \begin{align} \label{eq:system}
      \frac{d}{dt}
\begin{bmatrix}
     x(t) \\
     f(t)
     \end{bmatrix}
     =
     \begin{bmatrix}
     \psi(x(t) \mid c) \\
     -\divergence\,\psi(x(t) \mid c)
     \end{bmatrix}
     \end{align}
    
     with initial condition:
      
      \begin{align}
         \begin{bmatrix}
     x(0) \\
     f(0)
     \end{bmatrix}
     =
     \begin{bmatrix}
     y \\
     0
     \end{bmatrix}~.
     \end{align}

Where accumulates the divergence of the velocity field along the trajectory: $f(t) = \int_0^t \divergence \psi(x(t) \mid c)$ and hence $f(0)=0$.    
The system in \cref{eq:system} admits only one solution for all $t \in [0,1]$. \Cref{eq:integral} gives us that:
\begin{align}\label{eq:integral}
       \log p(x_0 \mid c) = \log p(x(1) \mid c) - f(1)~.
\end{align}
The probability $\log p(x(1) \mid c)$ is given directly by the distribution of the initial noise,a nd $f(1)$ is the solution of the system for $f$ at $t=1$.

\end{proof}

\paragraph{Extending Prop~1.}
Prop 1  can be extended to Riemenian Flow Matching simply by projecting the iterate onto the sphere at each step when iteratively solving the ODE \cref{eq:system}.
\\
For diffusion models, we do not have direct access to the velocity field. However, according to Song \etal \cite[Section D.2]{song2020score}, for a stochastic differential equation of the form:
\begin{equation}
    dx = f(x, t) dt + G(x, t) d\omega
\end{equation}
where $d\omega$ is a Wiener process \cite{durrett2019probability}, the velocity field $\Psi(x,t)$ can be expressed as:
\begin{align}
    v(x, t) &= f(x,t) - \frac{1}{2} \nabla \cdot [G(x,t) G(x,t)^T] \notag \\
    &\quad - \frac{1}{2} G(x,t) G(x,t)^T \nabla \log p_t(x_t \mid x_0, c)
    \label{eq:score}
\end{align}

In our case, we defined our forward noising process as:
\begin{equation}
    x_t = \sqrt{1-\kappa(t)} x_0 + \sqrt{\kappa(t)} \epsilon, \quad \epsilon \sim \mathcal{N}(0, I)~.
\end{equation}
This leads us to choose:
\begin{align}
    f(x,t) &= -\frac{1}{2} x \beta(t) \\
    G(x,t) &= \sqrt{\beta(t)}~,
\end{align}
where $\beta(t)$ represents the infinitesimal change in $x_t$ variation between $t$ and $t+\delta t$: $\beta(t) = x_{t+\delta t}-x_t$. 
According to \cite[Eq 29]{song2020score}, this process yields:
\begin{equation}
    x_t \sim \mathcal{N}\left(x_0 e^{-\frac{1}{2} \int_0^t \beta(s) ds}, \left(1-e^{-\int_0^t \beta(s) ds}\right) I\right)
\end{equation}
which implies that \cite[X]{song2020denoising}:
\begin{equation}
    \beta(t) = \frac{d\log(\kappa(t))}{dt}
\end{equation}

Finally, we can replace $\nabla \log p_t(x_t \mid x_0, c)$ with $-\epsilon_\theta(x_t, t, c)$ in \cref{eq:score}, as our model learns to predict the noise added to the data. This yields the following velocity field:
\begin{equation}
    \psi(x,t) = -\frac{1}{2}\beta(t)(x - \epsilon_\theta(x,t,c))~.
\end{equation}

\ifx 
In our case, given our forward noising process (Eq. 3-4), we have:
\begin{align}
    f(x,t) &= -\frac{1}{2} x \beta(t) \\
    G(x,t) &= \sqrt{\beta(t)}
\end{align}
where $\beta(t)$ represents the infinitesimal noise variation between $t$ and $t+\delta t$. 

We defined our forward noising process as:
\begin{equation}
    x_t = \sqrt{1-\kappa(t)} x_0 + \sqrt{\kappa(t)} \epsilon, \quad \epsilon \sim \mathcal{N}(0, I)
\end{equation}
where $\kappa(t)$ is a function controlling the noise level.
According to \cite{song2020score}, this process yields:
\begin{equation}
    x_t \sim \mathcal{N}\left(x_0 e^{-\frac{1}{2} \int_0^t \beta(s) ds}, \left(1-e^{-\int_0^t \beta(s) ds}\right) I\right)
\end{equation}
which implies:
\begin{equation}
    \beta(t) = \frac{d\log(\kappa(t))}{dt}
\end{equation}

Finally, we can replace $\nabla \log p_t(x_t \mid x_0, c)$ with $-\epsilon_\theta(x_t, t, c)$ in \cref{eq:score}, as our model learns to predict the noise added to the data. This yields the following velocity field:
\begin{equation}
    \Psi(x,t) = -\frac{1}{2}\beta(t)(x - \epsilon_\theta(x,t,c))
\end{equation}
\fi

%% file: figures/quali.tex
\begin{tabular}{ccc}
   \begin{subfigure}{0.3\textwidth}
    \includegraphics[width=\linewidth,height=.7\textwidth]{images/yfcc_med.jpg}
    \caption{Image from YFCC}
   \end{subfigure}
   &
   \begin{subfigure}{0.3\textwidth}
    \includegraphics[width=\linewidth,height=.7\textwidth]{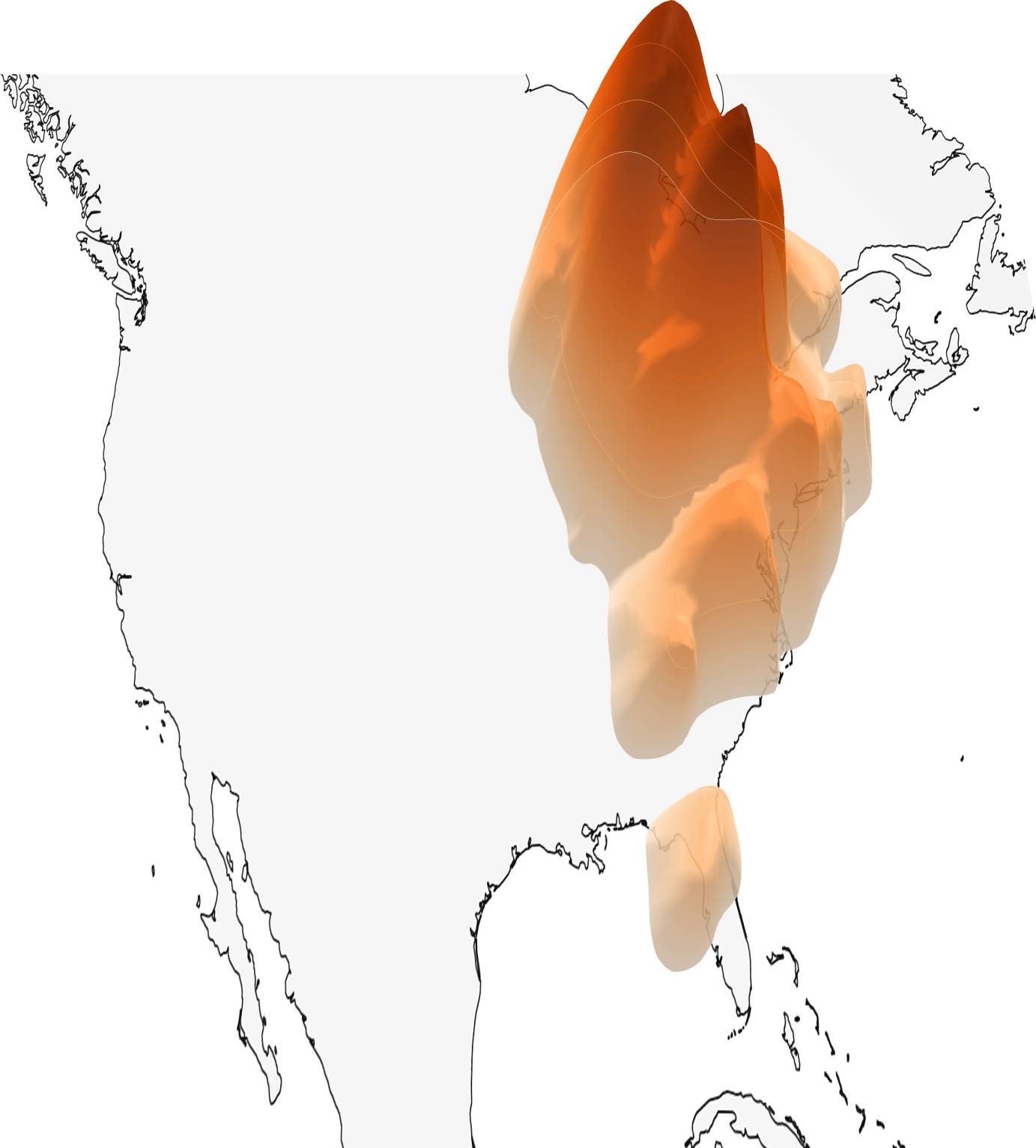}
    \caption{Diffusion in $\bR^3$}
   \end{subfigure}
   &
   \begin{subfigure}{0.3\textwidth}
    \includegraphics[width=\linewidth,height=.7\textwidth]{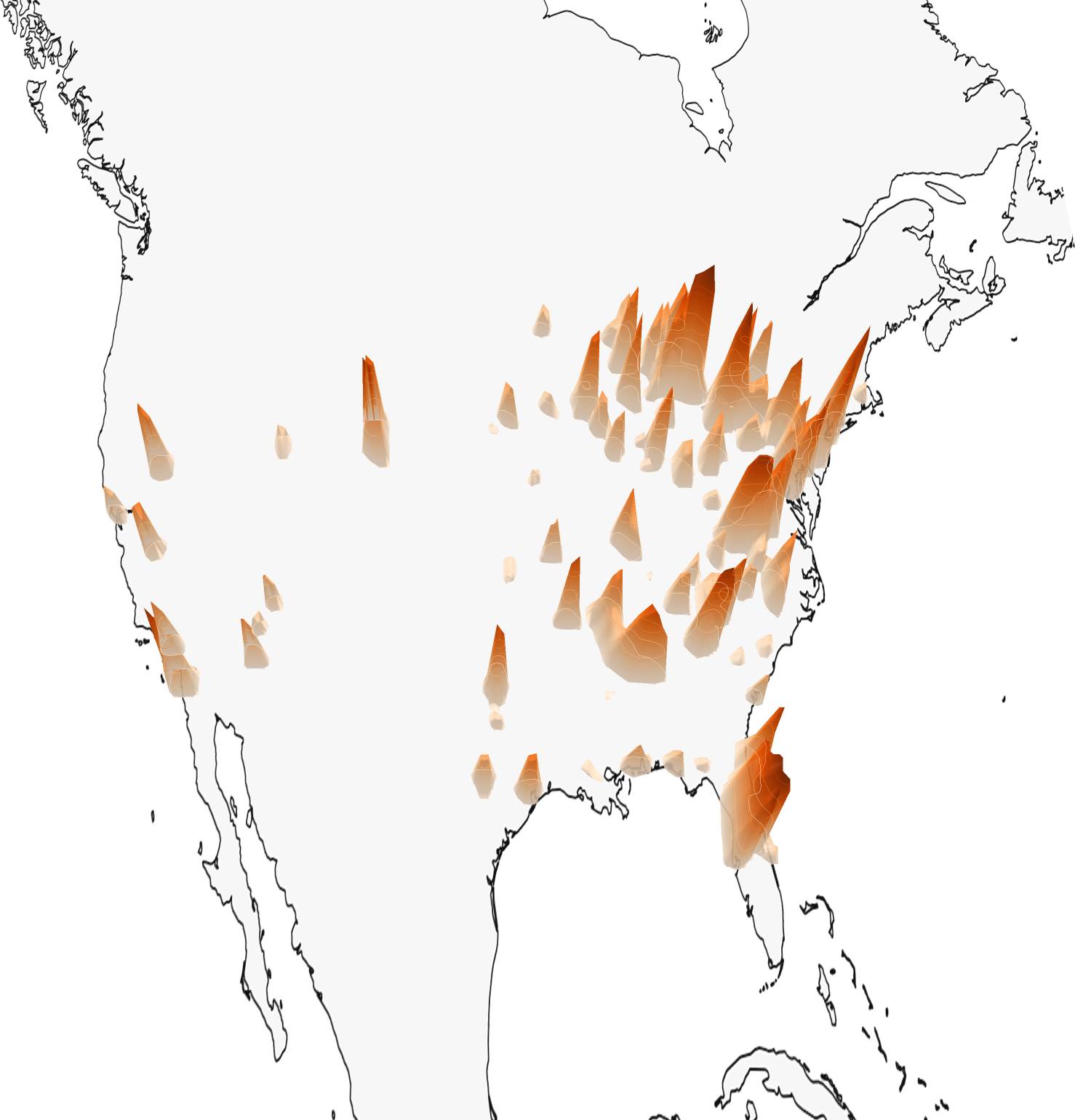}
    \caption{Flow Matching in $\bR^3$}
   \end{subfigure}
   \\
    \begin{subfigure}{0.3\textwidth}
    \includegraphics[width=\linewidth,height=.7\textwidth]{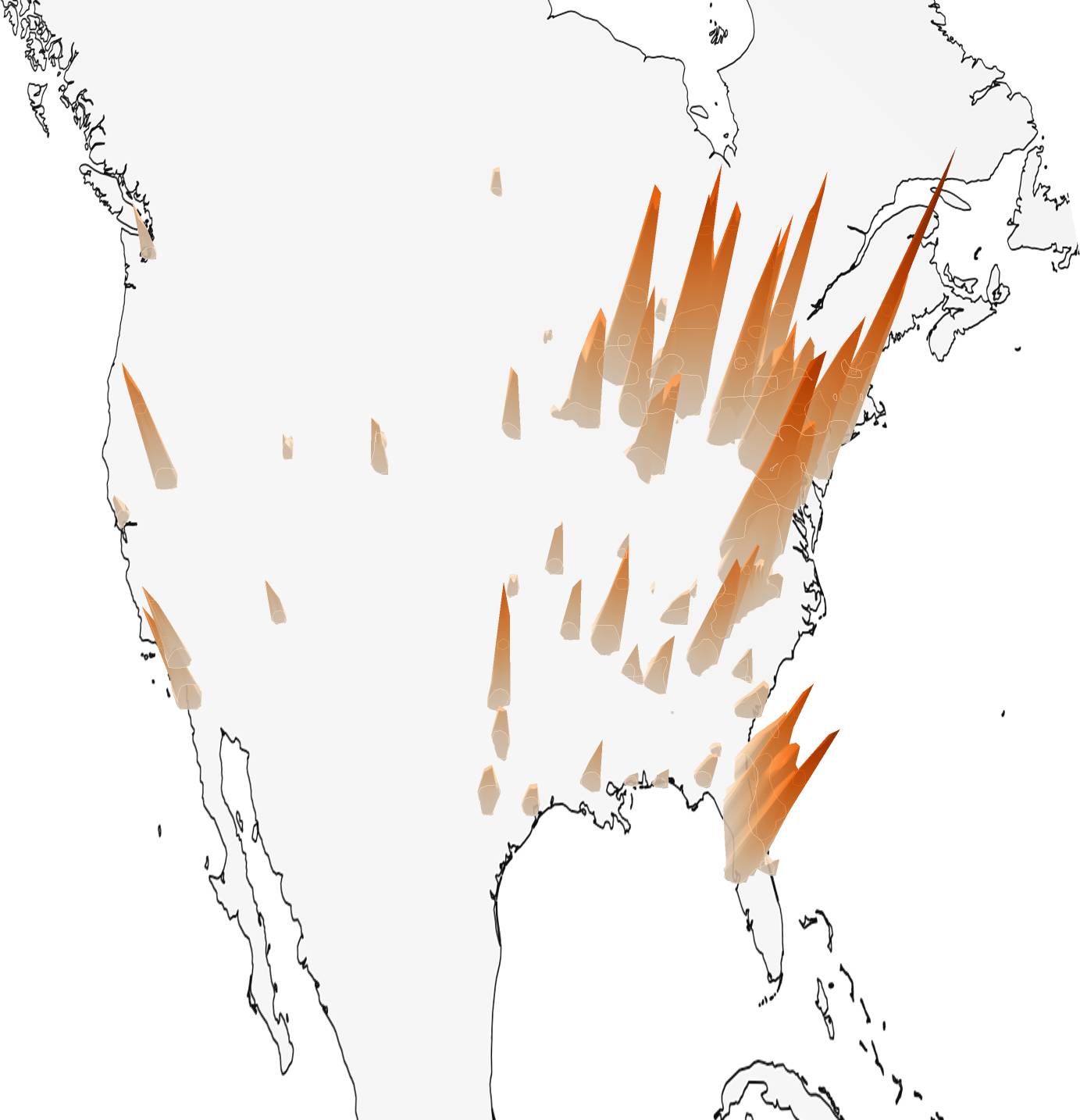}
    \caption{Riemannian Flow Matching in $\cS_2$}
   \end{subfigure}
   &
   \begin{subfigure}{0.3\textwidth}
    \includegraphics[width=\linewidth,height=.7\textwidth]{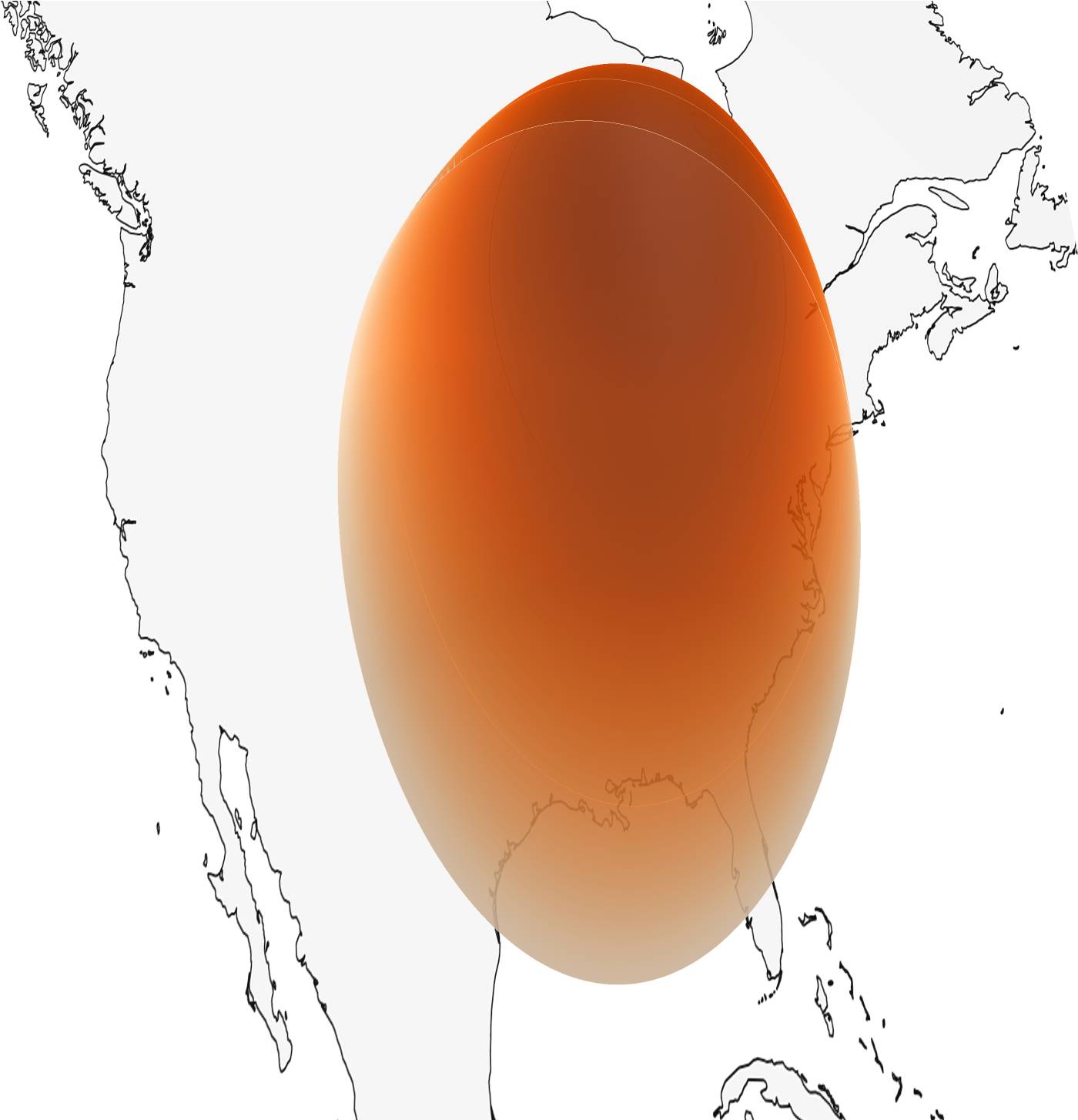}
    \caption{von Mishes-Fisher}
   \end{subfigure}
   &
   \begin{subfigure}{0.3\textwidth}
    \includegraphics[width=\linewidth,height=.7\textwidth]{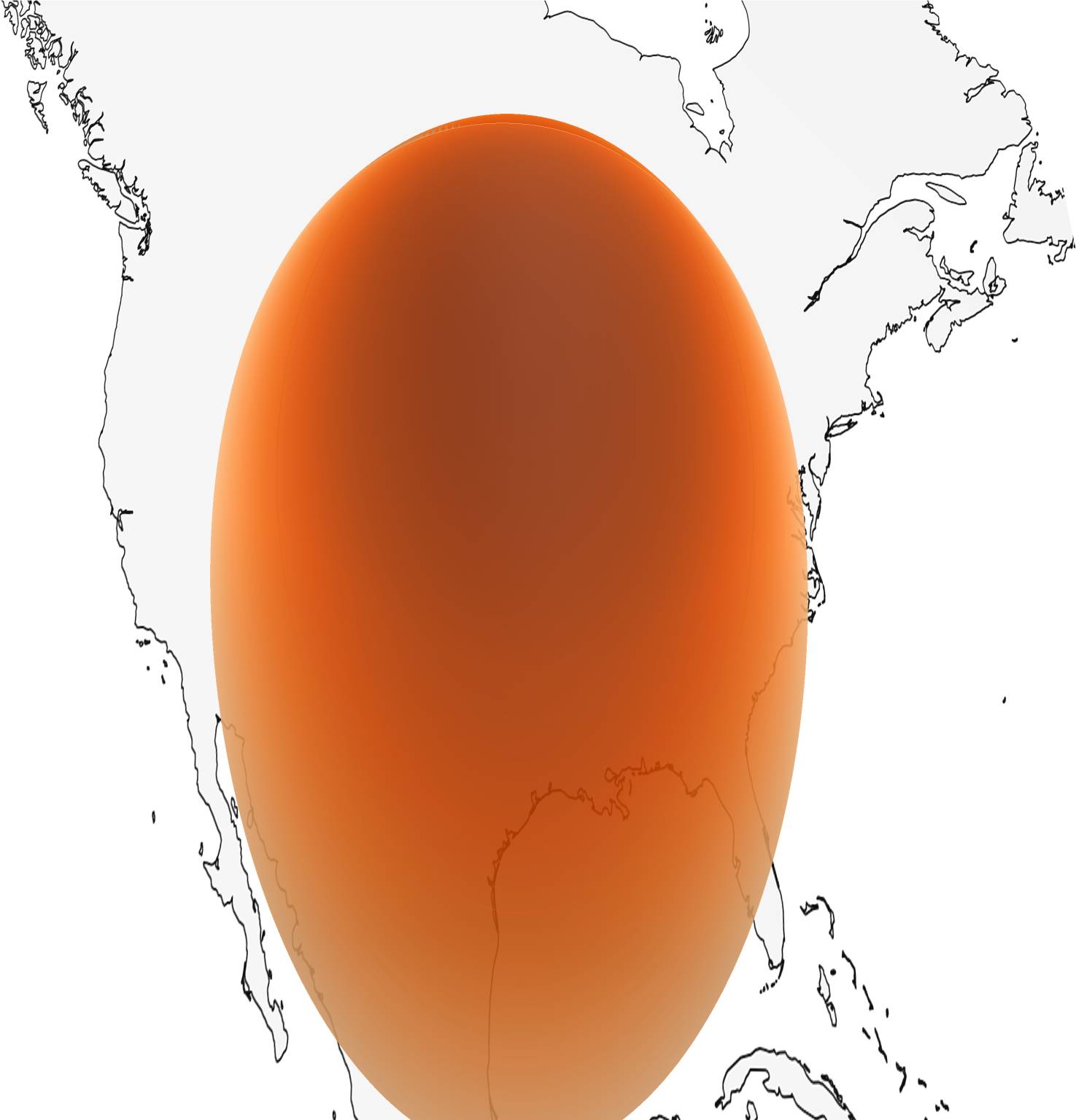}
    \caption{von Mishes-Fisher Mixture}
   \end{subfigure}
   
\end{tabular}

%% file: tables/ablation.tex
\small{
\begin{tabular}{l cccccc}
\toprule
& Geoscore $\uparrow$ & NLL $\downarrow$ & precision $\uparrow$ & recall $\uparrow$ & density $\uparrow$ & coverage $\uparrow$ \\
\midrule
Guided sampling & \underline{3746.79} & 33.1 & 0.841 & 0.896 & \underline{0.797} & \textbf{0.590} \\\greyrule
Single sampling & 3485.88 & \underline{-1.81} & \underline{0.844} & \underline{0.924} & 0.790 & 0.560 \\
Ensemble sampling & 3588.25 & \textbf{-4.31} & \textbf{0.899} & 0.785 & \textbf{0.881} & 0.537 \\\greyrule
Linear sigmoid & 3734.84 & -1.28 & 0.775 & \textbf{0.931} & 0.687 & 0.536 \\
Standard sigmoid & \textbf{3767.21} & -1.51 & 0.827 & 0.913 & 0.765 & \underline{0.565} \\\bottomrule
\end{tabular}
}

%% file: tables/detailed_density.tex
\footnotesize{
\begin{tabular}{l cccc cccc}
\toprule
& \multicolumn{4}{c}{OSV-5M} & \multicolumn{4}{c}{YFCC}  \\ \cmidrule(lr){2-5}\cmidrule(lr){6-9}
& precision $\uparrow$ & recall $\uparrow$ & density $\uparrow$ & coverage $\uparrow$ & precision $\uparrow$ & recall $\uparrow$ & density $\uparrow$ & coverage $\uparrow$ \\\midrule
Uniform & 0.29 & 0.98 & 0.21 & 0.21 & 0.59 & 0.99 & 0.38 & 0.22 \\
\rowcolor{gray!10} vMF Regression & 0.598 & \bf{0.982} & 0.499 & 0.446 & 0.667 & \bf{0.993} & 0.542 & 0.599 \\
vMF Mixture & 0.513 & \underline{0.980} & 0.422 & 0.358 & 0.626 & \underline{0.988} & 0.474 & 0.498 \\
\rowcolor{gray!10} \bf RFlowMatch $\cS_2$  (ours)  \bf & \underline{0.841} & 0.896 & \underline{0.797} & \bf{0.590} & \bf{0.957} & 0.952 & \bf{1.060} & \bf{0.926} \\
\greyrule
\bf Diffusion $\bR^3$ (ours) & 0.822 & 0.916 & 0.752 & 0.568 & 0.938 & 0.959 & 0.959 & 0.837 \\
\rowcolor{gray!10}\bf FlowMatch $\bR^3$ (ours) & \bf{0.845} & 0.907 & \bf{0.799} & \underline{0.575} & \underline{0.953} & 0.959 & \underline{1.037} & \underline{0.920} \\
\end{tabular}
}

%% file: figures/velocity.tex
\centering
\begin{tabular}{c@{\,}c@{\,}c}
     \begin{subfigure}{0.3\linewidth}
         \caption{No conditioning}
         \label{fig:velocity:a}
     \end{subfigure} 
     &
     \begin{subfigure}{0.3\linewidth}
\includegraphics[width=\linewidth, height=.6\linewidth]{example-image}
         \caption{$t=1$}
         \label{fig:velocity:a}
     \end{subfigure} 
     &
     \begin{subfigure}{0.3\linewidth}
\includegraphics[width=\linewidth, height=.6\linewidth]{example-image}
         \caption{$t=?$}
         \label{fig:velocity:a}
     \end{subfigure} 
     \\
          \begin{subfigure}{0.3\linewidth}
\includegraphics[width=\linewidth, height=.6\linewidth]{example-image}
         \caption{Image from Maryland}
         \label{fig:velocity:a}
     \end{subfigure} 
     &
     \begin{subfigure}{0.3\linewidth}
\includegraphics[width=\linewidth, height=.6\linewidth]{example-image}
         \caption{$t=1$}
         \label{fig:velocity:a}
     \end{subfigure} 
     &
     \begin{subfigure}{0.3\linewidth}
\includegraphics[width=\linewidth, height=.6\linewidth]{example-image}
         \caption{$t=1$}
         \label{fig:velocity:a}
     \end{subfigure} 
     \\
             \begin{subfigure}{0.3\linewidth}
\includegraphics[width=\linewidth, height=.6\linewidth]{example-image}
         \caption{Image from Nairobi}
         \label{fig:velocity:a}
     \end{subfigure} 
     &
     \begin{subfigure}{0.3\linewidth}
\includegraphics[width=\linewidth, height=.6\linewidth]{example-image}
         \caption{$t=1$}
         \label{fig:velocity:a}
     \end{subfigure} 
     &
     \begin{subfigure}{0.3\linewidth}
\includegraphics[width=\linewidth, height=.6\linewidth]{example-image}
         \caption{$t=?$}
         \label{fig:velocity:a}
     \end{subfigure} 
\end{tabular}

%% file: figures/architecture.tex
\def\yinput{0}
\def\yadaone{3}
\def\ygelu{5.5}
\def\ytimes{7.75}
\def\yplus{8.5}
\def\yadatwo{9.5}
\def\youtput{11.5}
\def\xcoordinate{0}
\def\ximage{5.75}
\def\xsilu{3.25}
\def\posencsize{0.7}

\definecolor{LINEARCOLOR}{RGB}{255,160,20}  
\definecolor{NORMCOLOR}{RGB}{255,223,70}  
\definecolor{ACTICOLOR}{RGB}{120,200,80} 

\begin{tikzpicture}

\tikzstyle{box}=[thick, text width=1.5cm,align=center, draw=black, rounded corners, anchor=center]

\node[] (inputcoordinate) at (\xcoordinate,\yinput) {$x_t$};

\node[box, fill=LINEARCOLOR] (linearcoordonate) at  (\xcoordinate,\yinput+1) {Linear};

 \draw[rounded corners, fill=gray!25, draw = none, ultra thick] (\xcoordinate-1.2, \yadaone-1) rectangle (\xsilu+1, \ygelu+1.75) {}; %
 \node [draw=none] at (\xsilu+0.5, \ygelu+1.35)  {$\times N$};

\node[box, fill=NORMCOLOR] (adaone) at  (\xcoordinate,\yadaone) {ADA-LN};

\node[box, fill=LINEARCOLOR] (lineargeluone) at  (\xcoordinate,\ygelu-1) {Linear};

\node[box, fill=ACTICOLOR] (gelu) at  (\xcoordinate,\ygelu) {GELU};

\node[box,  fill=LINEARCOLOR] (lineargelutwo) at  (\xcoordinate,\ygelu+1) {Linear};

\node[box,  fill=NORMCOLOR] (adatwo) at  (\xcoordinate,\yadatwo) {ADA-LN};

\node[box,  fill=LINEARCOLOR] (linearadatwo) at  (\xcoordinate,\yadatwo+1) {Linear};

\node[] (image) at (\ximage,\yinput) {$c$};

\node[] (time) at (\ximage-2,\yinput) {$\kappa(t)$};

\node[box,  fill=LINEARCOLOR] (linearimage) at  (\ximage,\yinput+1+\posencsize/2-0.1) {Linear};
    
    \node[draw, circle, minimum size=\posencsize cm, very thick, anchor=center] (posenc) at (\ximage-2.0,\yinput+1) {};
    \draw[domain=0:360,samples=100,smooth,variable=\x, thick] 
    plot({\ximage-2.0+.0+\posencsize/2+\x/360*\posencsize - \posencsize},{\yinput+1+sin(\x)*(0.3*\posencsize)});
    \draw[domain=0:360,samples=100,smooth,variable=\x, thick] 
    plot({\ximage-2.0+.0+\posencsize/2+\x/360*\posencsize - \posencsize},{\yinput+1+sin(2*\x)*(0.3*\posencsize)});
    
    \node[draw, circle, minimum size=\posencsize/2 cm, very thick, anchor=center] (plus) at (\ximage-1.0,\yinput+1.7) {};
    \draw[thick] (plus.north) -- (plus.south);
    \draw[thick] (plus.east) -- (plus.west);

\node[box,  fill=ACTICOLOR, rotate=90] (siluone) at  (\xsilu,\yadaone) {SILU};
\node[box, fill=LINEARCOLOR, rotate=90] (linearsiluone) at  (\xsilu-1,\yadaone) {Linear};

\node[box,  fill=ACTICOLOR, rotate=90] (silutwo) at  (\xsilu,\yadatwo) {SILU};
\node[box, fill=LINEARCOLOR, rotate=90] (linearsilutwo) at  (\xsilu-1,\yadatwo) {Linear};

\node[anchor=center] (outputcoordinate) at (\xcoordinate,\youtput) {};

\node [circle, draw=black, thick, minimum width=4mm] (times) at (\xcoordinate, \ytimes) {};
\fill (\xcoordinate, \ytimes) circle (0.5mm);

\node [circle, draw=black, thick, minimum width=4mm] (plus2) at (\xcoordinate, \yplus) {};
\draw [very thick] (\xcoordinate - 2mm, \yplus ) -- ++  (0.4,0) ++ (-0.2,0) -- ++ (0,0.2) -- ++ (0,-0.4) ;

\draw[->, very thick] (inputcoordinate) -- (linearcoordonate) -- (adaone);

\draw[->, very thick] (adaone) -- (lineargeluone) -- (gelu) -- (lineargelutwo) -- (times) -- (plus2) -- (adatwo) -- (linearadatwo) -- (outputcoordinate);

\draw[->, very thick] (time) -- (posenc) |- (plus);

\draw[->, very thick] (image) -- (linearimage) |- (plus);

\draw[->, very thick] (plus) |- (siluone) -- (linearsiluone) -- (adaone); 

\draw[->, very thick] (plus) |- (silutwo) -- (linearsilutwo) -- (adatwo); 

\draw[very thick, ->] (linearsiluone.north) ++ (-0.25,0) |- (times);

\draw [very thick, ->] (linearcoordonate) ++ (0,0.65) -- ++ (-1.25,0) |- (plus2);

\node[draw=none, above right=0mm and 1mm of linearcoordonate.north]  {$3 \mapsto d$};
\node[draw=none, above right=0mm and 1mm of lineargeluone.north]  {$d \mapsto 4d$};
\node[draw=none, above right=0mm and 1mm of lineargelutwo.north]  {$4d \mapsto d$};
\node[draw=none, above right=0mm and 1mm of linearadatwo.north]  {$d \mapsto 3$};

\node[draw=none, below left=0mm and 1mm of linearimage.south]  {$d \mapsto d$};
\node[draw=none, above left=0mm and 1mm of posenc.north]  {$d$};

\node[draw=none, below=0mm of linearsiluone.west]  {$d \mapsto 3d$};
\node[draw=none, above=0mm  of linearsilutwo.east]  {$d \mapsto 2d$};

\node[draw=none, below right=0mm and 0mm of adaone.east]  {$\alpha,\beta$};
\node[draw=none, above left=0mm and 0mm of linearsiluone.east]  {$\gamma$};
\node[draw=none, above right=0mm and 0mm of adatwo.east] {$\alpha,\beta$};

\end{tikzpicture}